\documentclass[twoside]{article}

\usepackage[accepted]{aistats2026}
\usepackage[utf8]{inputenc} 
\usepackage[T1]{fontenc}    
\usepackage{hyperref}       
\hypersetup{
   pdffitwindow=true,
   pdfstartview={FitH},
   pdfnewwindow=true,
   colorlinks,
   linktocpage=true,
   linkcolor=blue,
   urlcolor=blue,
   citecolor=blue
 }
 
\usepackage{url}            
\usepackage{booktabs}       
\usepackage{amsfonts}       
\usepackage{nicefrac}       
\usepackage{microtype}      
\usepackage{xcolor}  
\usepackage{latexsym}              
\usepackage{amsmath}               
\usepackage{amssymb}               
\usepackage{amsthm}
\usepackage[ruled,vlined,linesnumbered]{algorithm2e}
\usepackage{dsfont}
\usepackage{thmtools}

\newcommand{\Scal}{\ensuremath{\mathcal{S}}}

\newcommand{\Xcal}{\ensuremath{\mathcal{X}}}

\newcommand{\Ebb}{\ensuremath{\mathbb{E}}}

\newcommand{\Nbb}{\ensuremath{\mathbb{N}}}

\newcommand{\Rbb}{\ensuremath{\mathbb{R}}}

\newcommand{\LB}{\left[}
\newcommand{\RB}{\right]}
\newcommand{\LC}{\left\{}
\newcommand{\RC}{\right\}}

\newcommand{\LP}{\left(}
\newcommand{\RP}{\right)}

\newcommand{\ie}{\textit{i.e.}\xspace}

\newcommand{\eg}{\textit{e.g.}\xspace}

\newcommand{\wrt}{\textit{w.r.t.}\xspace}
\newcommand{\iid}{\textit{i.i.d.}\xspace}

\newcommand{\Q}{\mathrm{Q}}

\renewcommand{\S}{\Scal}

\usepackage{comment}
\usepackage{enumitem}

\usepackage{graphicx}

\usepackage{cleveref}
\crefname{assumption}{Assumption}{Assumptions}
\crefname{equation}{Eq.}{Eqs.}
\crefname{figure}{Fig.}{Figs.}
\crefname{table}{Table}{Tables}
\crefname{section}{Sec.}{Secs.}
\crefname{theorem}{Thm.}{Thms.}
\crefname{lemma}{Lemma}{Lemmas}
\crefname{corollary}{Cor.}{Cors.}
\crefname{example}{Example}{Examples}
\crefname{appendix}{Appendix}{Appendixes}
\crefname{remark}{Remark}{Remark}

\theoremstyle{plain}  
\newtheorem{theorem}{Theorem}[section]

\newtheorem{assumption}[theorem]{Assumption}

\newtheorem{lemma}[theorem]{Lemma}
\newtheorem{proposition}[theorem]{Proposition}

\newtheorem{corollary}[theorem]{Corollary}

\DeclareMathOperator*{\argmax}{arg\,max}
\DeclareMathOperator*{\argmin}{arg\,min}

\usepackage[authoryear]{natbib}
\usepackage{array}
\usepackage{multirow}
\usepackage{multicol}

\begin{document}

%

%

\twocolumn[


\aistatstitle{Sequential Off-Policy Learning with Logarithmic Smoothing}

\aistatsauthor{Maxime Haddouche\footnotemark[1] \And Otmane Sakhi\footnotemark[1]}

\aistatsaddress{ INRIA - CNRS, École Normale Supérieure,\\ PSL University, France \And Criteo AI Lab \\
Paris, France} ]
\footnotetext[1]{Equal contribution.}
\begin{abstract}

Off-policy learning enables training policies from logged interaction data. Most prior work considers the batch setting, where a policy is learned from data generated by a single behavior policy. In real systems, however, policies are updated and redeployed repeatedly, each time training on all previously collected data while generating new interactions for future updates. This sequential off-policy learning setting is common in practice but remains largely unexplored theoretically. In this work, we present and study a simple algorithm for \emph{sequential off-policy learning}, combining Logarithmic Smoothing (\texttt{LS}) estimation with online PAC-Bayesian tools. We further show that a principled adjustment to \texttt{LS} improves performance and accelerates convergence under mild conditions. The algorithms introduced generalize previous work: they match state-of-the-art offline approaches in the batch case and substantially outperform them when policies are updated sequentially. Empirical evaluations highlight both the benefits of the sequential framework and the strength of the proposed algorithms.
\end{abstract}

\section{INTRODUCTION}

Off-policy learning has become a ubiquitous task for practitioners in a wide range of applications such as recommender systems \citep{Sakhi_Rohde_Gilotte_2023}, ad placement \citep{bottou2013counterfactual}, and personalized medicine \citep{kallus2018policy}. This task involves learning optimal policies from offline data summarizing a \emph{behavior policy}'s past interactions. These interactions, typically captured as tuples of contextual features, action played, and reward received, hold valuable insight into the underlying dynamics of the environment. Each tuple represents a single round of interaction, where the agent observes a context (including relevant features), takes an action according to its current, behavior policy, and receives a reward that depends on both the observed context and the taken action. 

The off-policy learning literature focused on developing efficient approaches to learn from these logged interactions \citep{bottou2013counterfactual, swaminathan2015batch, london2019bayesian, sakhi2023pac}. At the core of these methods lies the development of tailored off-policy estimators with improved concentration properties \citep{bottou2013counterfactual, aouali23a, metelli2021subgaussian, sakhi2024logarithmic}, combined with tight generalization bounds allowing efficient learning \citep{swaminathan2015batch, sakhi2023pac, sakhi2024logarithmic}. The methodologies employed derive from the pessimism principle \citep{jin2021pessimism}, which has been demonstrated to be optimal for learning in these scenarios. Particularly, analyzing this learning problem from the PAC-Bayesian learning theory lens \citep{london2019bayesian} proved effective, improving the seminal work of \cite{swaminathan2015batch}. This direction was further explored by \cite{sakhi2023pac, aouali23a, pac_bayes_unified, gabbianelli2023importance, sakhi2024logarithmic}, reducing assumptions and improving the pessimistic principle, \emph{i.e.} controlling the worst-case scenario, with latest developments \citep{sakhi2024logarithmic} resulting in state of the art approaches. 

These methods are well understood and have been extensively studied in the one-step, batch setting. In practice, however, practitioners in production environments typically apply them \emph{sequentially} \citep{scrm}, repeatedly updating and redeploying policies by retraining on the full history of collected data while simultaneously generating new interactions for future updates. Intuitively, sequential off-policy learning is preferable, as it can yield improved policies at each deployment compared to batch approaches. This, in turn, enhances the quality of the logged data and accelerates convergence toward the optimal policy. Although this sequential off-policy learning paradigm is ubiquitous in real-world applications, it remains largely underexplored from a theoretical perspective. Recent work has begun to address this gap \citep{adaptive, scrm}, but existing algorithms deviate from common practice: they rely on strong uniform coverage assumptions and often require unrealistic data collection strategies \citep{scrm}. In the conventional off-policy learning setting, some of these assumptions are avoided via the use of PAC-Bayes learning \citep{alquier2024user} combined with novel off-policy estimators to reach both theory-driven learning algorithms and convergence guarantees \citep{sakhi2024logarithmic}. However, only batch algorithms have been developed, leaving the existence of sequential off-policy PAC-Bayes algorithms as an open question. In this work, we close this gap and develop the first sequential, off-policy PAC-Bayes algorithms with associated theoretical guarantees.

\textbf{Contributions.} We study the sequential deployment of state-of-the-art off-policy learning methods, a paradigm extensively used in real-world applications \citep{adaptive, scrm}. We extend in \Cref{sec: warmup-pls} the off-policy, logarithmic smoothing PAC-Bayesian strategy of \cite{sakhi2024logarithmic} to the sequential scenario, allowing the approach to handle successive interaction data, and progressively update the policy. This sequential strategy, detailed in \Cref{alg:extension-sakhi}, is theoretically analyzed and shown to benefit from improved performance compared to its offline counterpart.
Building on this, we identify a core limitation of our previous algorithm, inherited from \cite{sakhi2024logarithmic}. In \Cref{sec: accelerated-convergence}, we address this issue by introducing a second original sequential algorithm (\Cref{alg:s_analyzed}) enjoying accelerated convergence rates to the optimal policy under reasonable assumptions. Our sequential algorithms encompass the offline scenario and strictly generalize the state-of-the-art offline PAC-Bayesian approach.
Finally, we conduct in \Cref{sec: expes} numerical experiments to demonstrate the effectiveness of our approach, showing that they outperform both offline strategies and recent sequential strategies.

\textbf{Outline.} \Cref{sec: framework} introduces the sequential framework and the necessary background. 
\Cref{sec: warmup-pls} provably extends the logarithmic smoothing (LS) PAC-Bayesian off-policy learning approach to the sequential setting and studies its benefits and limitations. 
\Cref{sec: accelerated-convergence} provides a simple adjustment to the LS estimator, achieving an accelerated convergence rate under mild conditions. 
Empirical validation is conducted in \Cref{sec: expes} to highlight the favorable performance of our proposed method, and \Cref{sec: ccl} provides concluding remarks.




\section{FRAMEWORK}
\label{sec: framework}

\paragraph{Contextual bandits.} Let $\Xcal \subset \mathbb{R}^d$ be the context space, which is a compact subset of $\mathbb{R}^d$, and let $\mathcal{A}=[K]$ be a finite action set. An agent's actions are guided by a \emph{stochastic} policy $\pi: \mathcal{X} \rightarrow \mathcal{P}(\mathcal{A})$ within a policy space $\Pi$. Given a context $x \in \mathcal{X}, \pi(\cdot\mid x)$ is a probability distribution over the action set $\mathcal{A} ; \pi(a \mid x)$ is the probability that the agent selects action $a$ in context $x$.

\textbf{Sequential framework.} We consider that the agent interacts sequentially with a data stream: at time $k\geq 0$, the agent deploys a new behavior policy $\pi_k$ and a new batch of contexts $\{x_j\mid j\in I_k\}$ observed by the agent is revealed, where $|I_k|= n_k$ and all $(I_k)_{k\geq 0}$ are disjoint. These contexts are drawn \emph{i.i.d.} from the same, unknown distribution $\nu$. 

For each $j \in I_k$, the agent selects for context $x_j$ an action $a_j \sim \pi_{k}\left(\cdot \mid x_j\right)$, where $\pi_{k}$ is the current \emph{behavior} policy of the agent at time $k$, and receives a stochastic cost $c_j \in [-1,0]$ that depends on the observed context $x_j$ and the taken action $a_j$. This cost $c_j$ is sampled from an unknown cost distribution $p\left(\cdot | x_j, a_j\right)$ and we denote by $c(a,x)$ its expectation under $p\left(\cdot | x_j, a_j\right)$. After $|I_k| = n_k$ interactions, the agent sequentially updates its behavior to $\pi_{k+1}$  \wrt previous behavior policies $(\pi_j)_{j=0\cdots k}$ and is allowed to use all available data $\{x_i, a_i, c_i\}_{i=1\cdots N_k}$, where  $N_k := \sum_{i=0}^k n_i$. 

Formally, we define the full (countable) dataset $S := \cup_{k\in\Nbb} S_k$, where $S_k := \LC(x_j,a_j,c_j) \mid j\in I_k\RC$ is the set of logged data at time $k$. We also define $\mathcal{F}_k$ a filtration adapted to $\LP \cup_{j=0}^k S_j \RP_{j\in 0\cdots k}$ ($\mathcal{F}_{-1}$ is the trivial $\sigma$-algebra). We assume that, conditionally to $\mathcal{F}_{k-1}$, $S_k$ is constituted of mutually independent random variables. This generalizes the pure \emph{static} case, corresponding to $k=0$. We denote as \emph{sequential algorithm}, any algorithm $\texttt{Alg}$ such that its output $\pi_{k+1}$ at time $k$ is $\mathcal{F}_{k}$-measurable.  In other words, if we denote $S_{0:k} := \cup_{j =0}^k S_j$, then we have $\pi_{k+1}=\texttt{Alg}(S_{0:k})$.

To ease notations, we denote by $P(\pi) := \nu \times \pi \times p$ the distribution of $(x, a, c)$ tuples generated by policy $\pi$. 
The performance of any policy $\pi$ is assessed by its \emph{risk}, which is defined as the averaged incurred cost of $\pi$ over the draw of new context:
$$R(\pi)=\mathbb{E}_{(x, a, c) \sim P(\pi)}[c] = \Ebb_{x \sim \nu, a \sim \pi(\cdot \mid x)}[c(x,a)]\,.$$
Our goal is to come up with sequential learning algorithms that guarantee the convergence to the optimal policy $\pi_\star$, defined as:
$\pi_\star = \argmin_{\pi \in \Pi} R(\pi)\,.$ For any context $x$, the optimal action is defined as $a_\star(x) = \argmin_{a' \in \mathcal{A}}  c(a', x)$ with ties broken arbitrarily. Hence, the optimal policy $\pi_\star$ chooses deterministically the optimal action $a_\star(x)$ for every context $x$:
$$\forall (x, a) \in \mathcal{X} \times \mathcal{A}\,, \quad \pi_\star(a|x) = \mathds{1}\left[a = a_\star(x) \right].$$

\textbf{PAC-Bayesian off-policy learning.} State of the art off-policy learning algorithms \citep{sakhi2023pac, gabbianelli2023importance, sakhi2024logarithmic} employ the pessimism principle, shown to be optimal in the batch, off-policy learning setting \citep{jin2021pessimism}. These methods aim at minimizing an empirical, tight upper bound on the risk $R(\pi)$ over policies $\pi \in \Pi$. These upper bounds leverage PAC-Bayes learning \citep{alquier2024user} which controls the generalization guarantee of a \emph{posterior} distribution $Q$ over a predictor space $\Theta$ (namely $Q\in\mathcal{P}(\Theta)$). To make this consistent with controlling the risk of a policy $\pi$, we exploit a specific parametrization of policies $\pi \in \Pi$. For any $\theta \in \Theta$, we define the deterministic policy: $ \forall x, d_{\theta}(x) := \mathds{1}\{ a= f_{\theta}(x)\} $, where $f_{\theta}: \mathcal{X} \rightarrow \mathcal{A}$, a parametrized predictor. Then, each distribution $Q\in\mathcal{P}(\Theta)$ defines a policy $\pi_Q$ by the following mapping:
\begin{align*}
    \forall (x, a) \in \mathcal{X}\times\mathcal{A}\,, \quad \pi_Q(a|x) = \mathbb{E}_{\theta \sim Q} \left [ \mathds{1}\{ a= f_{\theta}(x)\}\right]\,.
\end{align*}
This expression is general as any policy can be represented in this form \cite[Theorem 3.1]{sakhi2023pac}. With this duality, \citet{sakhi2024logarithmic} gets with the PAC-Bayesian toolbox, by fixing a reference distribution $P$, and setting $\delta \in (0, 1]$ a learning bound of the following general form, holding with probability $1 - \delta$, simultaneously for all $Q$:
\begin{align*}
    \forall Q \in \mathcal{C}(\Theta)\,, R(\pi_Q) \le \hat{R}(\pi_Q) + \mathcal{O}\left( \operatorname{KL}(Q||P)\right)\,,
\end{align*}
upper bounding the true risk of $\pi_Q$ by a risk estimator $\hat{R}(\pi_Q)$ and $\operatorname{KL}(Q||P) := \Ebb_{Q}\left[ \log\left( \frac{dQ}{dP}\right) \right]$, the Kullback-Leibler divergence between $P$ and $Q$. In the batch setting, these upper bounds are used as learning objectives and the convergence guarantees of their learned policies are studied \citep{gabbianelli2023importance, sakhi2023pac}. We extend this paradigm to the sequential framework, and begin by defining off-policy estimators using the interaction data collected.


\textbf{Regularized IPS.} 
Given the successive arrival of data batches $(S_k)_{k\ge 0}$ and output of sequential algorithm $(\pi_k)_{k\geq 0}$, we define for any $\pi\in\Pi$ the \emph{Inverse Propensity Scoring} (IPS) estimator \citep{horvitz1952generalization} as:
$\hat{R}_{0:k}^{\textsc{IPS}}(\pi) = \frac{1}{N_k} \sum_{j = 0}^k\sum_{i \in I_j} \frac{\pi(a_i|x_i)}{\pi_j(a_i|x_i)}c_i\,,$
encompassing the static case ($k = 0$). 
This estimator is unbiased under the common support assumption \citep{mcbook}. However, as in the static case, this estimator has poor concentration properties \citep{metelli2021subgaussian} and regularized variants, trading some bias for lower variance are preferred \citep{bottou2013counterfactual, optimistic_shrinkage}. The majority of these regularized estimators stems from specific choices of a regularizer $h: [0,1]^2 \times [-1, 0] \rightarrow \mathbb{R}^-$ in the following definition:
\begin{align*}
    \hat{R}_{0:k}^{h}(\pi) := \frac{1}{N_k} \sum_{j = 0}^k\sum_{i \in I_j} h(\pi(a_i|x_i), \pi_j(a_i|x_i), c_i)\,,
\end{align*}
with $h$ respecting $\forall (p, q, c) \in [0,1]^2 \times [-1, 0]\,, \quad pc/q \le h(p, q, c) \le 0$. $h(p,q,c) = pc/q$ recovers the IPS estimator while other choices recover known, well behaved estimators such as clipped IPS \citep{bottou2013counterfactual}, power-mean estimators \citep{metelli2021subgaussian} and Implicit-Exploration \citep{gabbianelli2023importance} to name a few. When $k=0$, we denote $\hat{R}_{0:0}^{h}$ as $\hat{R}_{0}^{h}$. 




\section{WARM-UP ON PESSIMISTIC LOGARITHMIC SMOOTHING}
\label{sec: warmup-pls}
\subsection{Background: Offline algorithms with theoretical guarantees}
PAC-Bayesian Logarithmic Smoothing recently appeared as a fruitful theoretical approach to obtain both concrete learning algorithms and generalization guarantees for contextual bandits. 
In \cite{sakhi2024logarithmic}, PAC-Bayesian bounds have been developed for off-policy contextual bandits, and optimizing them provide novel learning algorithms with optimality guarantees. We first recall the definition of the logarithmic smoothing estimator (\texttt{LS}) as well as its associated regularizer. For any $\lambda >0$ we define \texttt{LS} for $k\geq 0$,
\begin{align*}
    &h_{\mathrm{LS}}^{\lambda\mathrm{-Lin}}(p,q,c) := -\frac{p}{\lambda} \log\left( 1-\frac{\lambda c}{q}\right) \,,\, 
    \\
    &\hat{R}_{0:k}^{\lambda\mathrm{-LS}} := \hat{R}_{0:k}^{h_{\mathrm{LS}}^{\lambda\mathrm{-Lin}}}.
\end{align*}
Using these notions, \cite{sakhi2024logarithmic} obtained results for static off-policy learning \emph{i.e.} data are gathered in a single batch $S_0$ of size $n_0$. We recall them below.

\begin{proposition}[Propositions 10 and 11 of \citealp{sakhi2024logarithmic}]
\label{prop: batch-pls}
Given a prior $P \in \mathcal{P}(\Theta), \delta \in(0,1]$ and $\lambda>0$, the following holds with probability at least $1-\delta$ over $S_0,\forall Q \in \mathcal{P}(\Theta),$ :
\[ R\left(\pi_Q\right) \leq \hat{R}_{0}^{\lambda-\mathrm{LS}}\left(\pi_Q\right)+\frac{\operatorname{KL}(Q \| P)+\log \frac{1}{\delta}}{\lambda n_0}.\]
Furthermore, for any $\mathcal{C}(\Theta) \subset  \mathcal{P}(\Theta)$, denoting by $\hat{Q} \in \operatorname{argmin}_{Q \in \mathcal{C}(\Theta)} \hat{R}_{0}^{\lambda-\mathrm{LS}}\left(\pi_Q\right)+\frac{\operatorname{KL}(Q \| P)}{\lambda n_0}$ and $Q^\star\in \operatorname{argmin}_{Q \in \mathcal{C}(\Theta)} R(\pi_Q)$,  we have: 
\begin{align*}
     R(\pi_{\hat{Q}}) - R(\pi_{Q^\star})  \leq \lambda \mathcal{S}(\pi_{Q^\star}) + \frac{2\operatorname{KL}(Q^\star \| P)+\log \frac{2}{\delta}}{\lambda n_0},
\end{align*}
    where $\mathcal{S}(\pi) := \Ebb_{(x, a, c) \sim P(\pi_0)}\left[ \frac{\pi(a\mid x) c^2}{\pi_0^2(a\mid x)} \right]$ is a pseudo-variance term.
\end{proposition}

Those results provide two distinct insights: {\it (i)}, the learning procedure yielding $\hat{Q}$ (and exploited in practice in \cite{sakhi2024logarithmic}) minimizes a generalization bound and {\it (ii)}, once $\mathcal{S}(\pi_{Q^\star})$ finite, taking $\lambda = 1/\sqrt{n_0}$ gives a convergence rate of $1/\sqrt{n_0}$ to the optimal value of the risk. However, \Cref{prop: batch-pls} and its associated learning algorithm suffer from theoretical and practical limitations. 
\begin{enumerate}[left=0pt, itemsep=0pt, topsep=0pt]
    \item A batch procedure cannot handle a progressive arrival of data $(S_k)_{k\geq 0}$. In particular, \Cref{prop: batch-pls} does not allow to retrain our policy \emph{w.r.t} all past data when a new data batch arrives.
    \item The pseudo-variance term of \Cref{prop: batch-pls} does not go to zero, even in the oracle case where $\pi_0 = \pi_{Q^\star}$, whatever the optimal policy $\pi_{Q^\star}$ is. This issue is not an artifact of the proof technique, but from the performance objective $R(\pi)$, and consequently its estimator. This exhibits a fundamental limitation of this procedure as the slow convergence rate of $1/\sqrt{n_0}$ is inevitable.
\end{enumerate}
We answer those two fundamental issues below. To answer 1, we show in \Cref{sec: first-adaptive-alg} that leveraging online PAC-Bayes techniques allows obtaining a sequential version of \Cref{prop: batch-pls}, permitting both the arrival of new data on-the-fly and naturally re-using former ones. Then, we answer 2, in \Cref{sec: accelerated-convergence} by adjusting the performance objective and its respective estimator, providing a novel sequential learning algorithm with accelerated convergence.

\subsection{Towards a first sequential algorithm}
\label{sec: first-adaptive-alg}
In this part, we assume, for pedagogical concerns, that all datasets $S_k$ have the same size $n_k=m$. Still exploiting Logarithmic Smoothing, we provide a sequential counterpart of \Cref{prop: batch-pls} by involving online PAC-Bayes learning techniques \citep{haddouche2022online,haddouche2023pac}. 
Our result is stated as follows: 

\begin{restatable}{theorem}{extensionsakhi}
    \label{th: extension-sakhi}
    For any sequential algorithm \textup{\texttt{Alg}}, any $m>0$ and data independent prior $P$. For any $\lambda >0, \delta \in (0,1]$, with probability at least $1-\delta$ over $S$, we have, for all $k\geq 0,$ denoting by $\pi_{k+1}$ the output of \textup{\texttt{Alg}} at time $k$, for all $Q\in\mathcal{P}(\Theta)$: 
\begin{align*}
     R(\pi_{Q}) \leq \hat{R}_{0:k}^{\lambda-\mathrm{LS}}(\pi_{\Q}) + \frac{\operatorname{KL}(Q,P) + \log \left( \frac{1}{\delta}  \right)}{\lambda (k+1)m}.
\end{align*}
\end{restatable}

The proof lies in \Cref{sec: extensions-section-3}. Note that taking $k=0$ and $m=n_0$ retrieves exactly the batch case of \Cref{prop: batch-pls}. On the contrary taking $k=n$ and $m=1$ provides a purely online procedure. Our theorem provides an interpolation between batch and online scenarios, allowing sub-batches of data to arrive at each time step and optimize on-the-fly.

\RestyleAlgo{ruled}
\begin{algorithm}
\caption{\textbf{Sequential Policy Learning via Logarithmic Smoothing}}\label{alg:extension-sakhi}
\textbf{Input}: Policy $\pi_0$, policy class $\mathcal{C}(\Theta)$, $\lambda >0$. \\
Initialize $P$ such that $\pi_{P} = \pi_0$ and $\mathcal{S}_{-1} = \{ \}$. \\
\For{$k \geq 0$}{
Collect data $S_k$ of size $|I_k| = m$ with $\pi_{k}$ and define global collection $S_{0:k} = S_{0:k - 1} \cup S_k$.\\
Solve the optimization problem:
\begin{align*}
\hspace{-2mm}\hat{Q}_{k+1} = \operatorname{argmin}_{Q\in\mathcal{C}(\Theta)} \left\{\hat{R}_{0:k}^{\lambda-\mathrm{LS}}(\pi_{\Q}) + \frac{\operatorname{KL}(Q,P)}{\lambda (k+1)m} \right\}
\end{align*}
Set $\pi_{k+1} = \pi_{\hat{Q}_{k+1}}$.
}
\end{algorithm}

While \Cref{th: extension-sakhi} has a similar shape to \Cref{prop: batch-pls}, the involved toolbox differs fundamentally. Indeed, while batch PAC-Bayes only requires controlling an exponential moment via Markov's inequality, Online PAC-Bayes requires a careful design and control of a supermartingale through Ville's inequality \citep{doob1939jean}. This core difference allows, beyond \Cref{th: extension-sakhi}, to reach a more general result, stated in \Cref{sec: non-iid-contexts}, and allowing \emph{any arbitrary evolution of the context distribution}. In particular, we can avoid the \iid assumption on contexts made in \Cref{sec: framework}. This better reflects real-world applications, where user behavior and responses naturally evolve over time. Although analyzing algorithm convergence in this non-stationary setting is of interest, it falls outside the scope of the present work.

Notice that the optimization problem $\operatorname{argmin}_{Q\in\mathcal{P}(\Theta)} \hat{R}_{0:k}^{\lambda-\mathrm{LS}}(\pi_{Q}) + \frac{\operatorname{KL}(Q,P)}{\lambda (k+1)m}$ admits a closed form solution: the \emph{Gibbs posterior} $\Q_\lambda$ such that $dQ_\lambda(\theta) \propto \exp\left( -\lambda(k+1)m \hat{R}_{0:k}^{\lambda-\mathrm{LS}}(d_\theta) \right)$ \cite[Theorem 1.2.6]{catoni2007pac}. However, Gibbs posteriors often require costly, advanced Monte Carlo methods \citep{smc_book} to be computed. A common practice is to restrict the optimization problem on a subclass $\mathcal{C}(\Theta) \subset \mathcal{P}(\Theta)$ on which the optimization problem is tractable, \eg the class of Gaussian distributions \citep{dziugaite2017computing,perez2021tighter}. Restricting the family of distributions is not limiting: as long as the family of parametrized predictors $\{ f_\theta \mid \theta\in \Theta\}$ is expressive enough, the induced policy family still contains the optimal policy \citep{sakhi2023pac}.

Finally, a key remark is that, given a fixed batch size $m>0$ and parameter $\lambda>0$, \Cref{th: extension-sakhi} holds for a countable number $k$ of sub-batches simultaneously. A fundamental difference with \Cref{prop: batch-pls} is that we allow the empirical risk $\hat{R}_{0:k}^{\lambda-\mathrm{LS}}$ to evolve with $k$, while the batch approach only tackles a single $S_0$. This allows the optimization of the right-hand-side of \Cref{th: extension-sakhi} at each time step and yields a sequential learning algorithm given in \Cref{alg:extension-sakhi}.


In \Cref{alg:extension-sakhi}, at time $k$, we instantiate $\hat{R}_{0:k}^{\lambda-\mathrm{LS}}$ with $(\pi_{j})_{j=0\cdots k}$ the outputs of \Cref{alg:extension-sakhi} up to time $k-1$. Similarly to the batch case, \Cref{alg:extension-sakhi} also enjoys a convergence guarantee to the optimal policy.
\begin{restatable}{proposition}{slowconvergence}
\label{prop: slowconvergence}
    For any fixed subbatch size $m>0$, any data independent prior $P$, any $\lambda >0, \delta \in (0,1]$, with probability at least $1-\delta$, for any $k\geq0$, denoting by $\pi_{k+1}$ the output of \Cref{alg:extension-sakhi} at time $k$ and $\pi_\star= \pi_{Q_\star}$ with $Q_\star\in \operatorname{argmin}_{Q\in \mathcal{C}(\Theta)} R(\pi_Q)$:
\begin{multline*}
     R\LP \pi_{k+1} \RP - R\LP \pi_{\star} \RP \\ \leq \frac{\lambda}{k+1} \sum_{j = 0}^k\Scal_{j}\LP \pi_{\star}\RP + 2\frac{ \operatorname{KL}(Q^\star,P) + \log \left( \frac{2}{\delta}  \right)}{\lambda (k+1)m} 
\end{multline*}
where $\Scal_{j}(\pi) :=  \Ebb_{(x, a, c) \sim P(\pi_j)}\LB   \frac{\pi(a \mid x) c^2}{\pi_j^2(a \mid x)}\RB$.
\end{restatable}
\Cref{prop: slowconvergence} is proven in \Cref{sec: extensions-section-3}. To obtain explicit convergence rates, we make the assumption that our policy class covers the optimal actions, i.e.:
\begin{assumption}[Coverage of the optimal actions]\label{assumption} For any policy $\pi$ in our policy class $\mathcal{C}(\Theta)$, we assume that: $\inf_x\pi(a^\star(x)|x) \ge C^\star$ for some constant $C^\star > 0$.
\end{assumption}
\Cref{assumption} only requires coverage of the optimal actions, which is strictly weaker than the uniform coverage condition assumed in prior work \citep{scrm}. With this assumption, we ensure that $S_j(\pi_\star)$ is finite and uniformly upper bounded by $1/C^\star$ for all $j$. Taking $\lambda= \frac{1}{\sqrt{(k+1)m}}$ yields an overall convergence rate of $\mathcal{O}\LP \frac{1}{\sqrt{(k+1)m}}\RP$, matching the rate of \Cref{prop: batch-pls} in the batch case (i.e $k=0, m=n_0$) and yielding an original rate of $\mathcal{O}(1/\sqrt{n})$ for an online setting ($k=n,m=1$). \Cref{prop: slowconvergence} also works for any intermediate regime between batch and online settings, showing the strength of the sequential approach. Looking at \Cref{prop: slowconvergence}, we observe that the upper bound depends on the average $\frac{1}{k+1}\sum_{j = 0}^k \mathcal{S}_j(\pi_\star)$. Intuitively, this average is expected to be smaller than $\mathcal{S}_0(\pi_\star)$, as intermediate policies $\pi_j$ should be progressively closer to $\pi_\star$. However, this algorithm still suffers from the same limitations to the static case here as the pseudo-variance terms $\mathcal{S}_{j}$ do not decrease to zero, even in the perfect situation where for all $j$, $\pi_j=\pi_{\star}$. 

In this part, we showed that online PAC-Bayes techniques successfully extend offline algorithms to the sequential setting, revealing the dependence on the data-collection process while enjoying the same convergence rates. A key question is then: is it possible to find a sequential algorithm with a faster convergence rate? 


\section{SEQUENTIAL LOGARITHMIC SMOOTHING WITH ACCELERATED CONVERGENCE}
\label{sec: accelerated-convergence}

\RestyleAlgo{ruled}
\begin{algorithm}
\caption{\textbf{Sequential Policy Learning via Adjusted 
 Logarithmic Smoothing}}\label{alg:s_analyzed}
\textbf{Input}: Policy $\pi_0$, a policy class $\mathcal{C}(\Theta)$, $\lambda\in (0,1)$. \\
Initialize $P$ such that $\pi_{P} = \pi_0$ and $\mathcal{S}_{-1} = \{ \}$. \\
\For{$k \geq0$:}{
Collect data $S_k$ of size $|I_k| = n_k$ with $\pi_{k}$. \\
Define global collection $S_{0:k} = S_{0:k - 1} \cup S_k$, and  $N_k := |S_{0:k}|$.\\
Solve the following optimization problem:
\begin{align*}
    \hat{Q}_{k+1} = \operatorname{argmin}_{Q\in\mathcal{C}} \left\{\hat{R}^{\lambda-\mathrm{adj}}_{0:k}(\pi_{\Q}) + \frac{\operatorname{KL}(Q,P)}{\lambda N_k} \right\}
\end{align*}
Set $\pi_{k+1} = \pi_{\hat{Q}_{k+1}}$.
}
\end{algorithm}

While \Cref{th: extension-sakhi} provides a natural sequential extension of \cite{sakhi2024logarithmic}, \Cref{alg:extension-sakhi} still suffers (\Cref{prop: slowconvergence}) from a slow convergence rate of $\mathcal{O} (1/\sqrt{km})$ due to the nature of the performance objective, and consequently the pseudo-variance term. Is it then possible to avoid such a term? The answer is yes, by adjusting $h^{\lambda - \mathrm{Lin}}$ introduced in \Cref{sec: warmup-pls}: we define $h^{\text{adj}}_\lambda(p,q,c):= -\frac{p}{\lambda} \log\left(1- \frac{\lambda c}{q(1+\lambda c)}\right)$ and denote it as the \emph{adjusted linear LS regularizer}.
We then define the adjusted risk for any sequential algorithm $\texttt{Alg}$ and $\lambda \in (0, 1), k\geq 0$:
\begin{align*}
\hat{R}^{\lambda-\mathrm{adj}}_{0:k}(\pi)
&:= R_{0,k}^{h^{\text{adj}}_\lambda}(\pi) \\
&= - \frac{1}{N_k}\sum_{j = 0}^k \sum_{i \in I_j}
\frac{\pi(a_i|x_i)}{\lambda} \\
&\quad \cdot \log\left(
1 - \frac{\lambda c_i}{\pi_{j}(a_i|x_i) (1 + \lambda c_i)}
\right)\,,
\end{align*}
where again, $\pi_{k+1}$ is the output of $\texttt{Alg}$ at time $k$. 
Using similar tools to those in \Cref{sec: first-adaptive-alg}, we are able to derive the following theoretical guarantee for our novel adjusted regularizer $h^{\text{adj}}_\lambda$.

\begin{theorem}
    \label{th: main-upperbound}
    For any sequential algorithm $\texttt{Alg}$, any sequence $(n_k)_{k\geq 0}$ and data independent prior $P$. For any $\boldsymbol{\lambda \in(0, 1)}, \delta \in (0,1]$, with probability at least $1-\delta$ over $S$, we have, for all $k\geq 0$, denoting by $\pi_{k+1}$ the output of $\texttt{Alg}$ at time $k$, for all $ Q\in\mathcal{P}(\Theta)$: 
\begin{multline*}
     R\LP \pi_{Q} \RP - \sum_{j = 0}^k \frac{n_j}{N_k} R\LP \pi_j \RP\\ \leq \hat{R}^{\lambda-\mathrm{adj}}_{0:k}(\pi_Q) + \hat{C}_{0,k}(\lambda) + \frac{ \operatorname{KL}(Q,P) + \log \left( \frac{1}{\delta}  \right)}{\lambda N_k}, 
\end{multline*}
with $\hat{C}_{0,k}(\lambda) = \frac{1}{N_k}\sum_{j = 0}^k\sum_{i \in I_j} \frac{1}{\lambda} \log(\frac{1}{1 + \lambda c_i})$.
\end{theorem}

The proof of \Cref{th: main-upperbound} is established in \Cref{sec: proofs-sec-4}. \Cref{th: main-upperbound} relies on the same analytical tools as \Cref{th: extension-sakhi}, and therefore inherits the same general insights. Notably, the result remains valid for varying batch sizes $(n_k)_{k \geq 0}$, which makes it applicable to real-world scenarios where data may arrive in irregular quantities. This theorem serves as the theoretical foundation of the main contribution of this work: a sequential learning algorithm for off-policy contextual bandits, presented in \Cref{alg:s_analyzed}.

Note that at time $k$, we instantiate $\hat{R}^{\lambda-\mathrm{adj}}_{0:k}(\pi_Q)$ with $(\pi_{j})_{j=0\cdots k}$ the previous outputs of \Cref{alg:s_analyzed}. 
\\
Note that in \Cref{th: main-upperbound}, contrary to \Cref{th: extension-sakhi}, we do not control directly $R(\pi_Q)$ but the difference $ R\LP \pi_{Q} \RP - \sum_{j = 0}^k \frac{n_j}{N_k} R\LP \pi_j \RP$. This means that running \Cref{alg:s_analyzed} at time $k$ using $\cup_{j=0}^k S_j$ tends to ensure a better generalisation ability than what we obtained (on a weighted average) by running \Cref{alg:s_analyzed} on $\cup_{j=0}^{k-1} S_j$. This phenomenon gives an intuition about why \Cref{alg:s_analyzed} converges faster to the optimum. We prove this intuition in what follows. 

\textbf{Accelerated convergence rate for \Cref{alg:s_analyzed}.} To prove the fast convergence of \Cref{alg:s_analyzed} to $\pi_\star$, we state a first proposition.


\begin{proposition}
    \label{prop: suboptimality_with_variance}
    For any data independent prior $P$, any $\boldsymbol{\lambda \in(0, 1)}, \delta \in (0,1]$, with probability at least $1-\delta$, we have for all $k\geq 0$, denoting by $\pi_{k+1}$ the output of \Cref{alg:s_analyzed} at time $k$: 
\begin{multline*}
      0 \le R\LP \pi_{k +1} \RP - R\LP \pi_{\star} \RP \\ \leq  \frac{\lambda}{1 - \lambda} \sum_{j = 0}^k\frac{n_j}{N_k} L\LP \pi_\star, \pi_j \RP  + 2\frac{ \operatorname{KL}(Q^\star,P) + \log \left( \frac{2}{\delta}  \right)}{\lambda N_k},   
\end{multline*}
where for all $\pi$,\\
$L\LP \pi, \pi_j \RP = \mathbb{E}_{x\sim \nu}\left[ \mathbb{E}_{\pi_j}\left[c^2 \right] + \mathbb{E}_{\pi}\left[c^2\left(\frac{1}{\pi_j(a|x)}
 - 2\right) \right] \right]$.
\end{proposition}
\Cref{prop: suboptimality_with_variance} is proven in \Cref{sec: proofs-sec-4} and highlights the interest of $h^{\text{adj}}_\lambda$: we avoid the pseudo variance terms $\mathcal{S}_{j}(\pi_\star)$ of \Cref{prop: slowconvergence} to get $L\LP \pi_\star, \pi_j \RP$. Contrary to $S_j(\pi_\star)$, we have that, for any deterministic policy $\pi_d$, $L(\pi_d, \pi_j) \rightarrow 0$ if $\pi_j \rightarrow \pi_d$. Such desirable behavior leaves room for accelerated convergence rates. To investigate this route, we target the convergence towards the optimal policy $\pi_\star$ and study the behavior of $L$. 

To prove the acceleration of our method, we need to define additional quantities. We define the sub-optimality of action $a$, given a context $x$ by: $\Delta_{a,x} := c(a, x) - c(a_\star(x),x) \ge 0$. Then, for any $u \in [0, 1)$, we define $\Delta_u$ such that:
\[P_x\left(\min_{a \neq a^\star(x)} \Delta_{a,x} \ge \Delta_u \right) = 1-u.\]
For a fixed $u$, $\Delta_u$ quantifies how separable $a^\star(x)$ is from the rest of actions. In other words, the magnitude of $\Delta_u$ quantifies the difficulty of the problem. $\Delta_u = 0$ for all $u \in [0, 1)$ means that the optimal action $a^\star(x)$ is not distinguishable from the second best action,  almost surely. Oppositely, $\Delta_u = 1$ for a specific $u$ means that with probability $1 - u$, the costs are deterministic and the optimal action $a^\star(x)$ is the only one to have nonzero cost (of $-1$). We want to avoid the first case, as the existence of a strictly positive $\Delta_u > 0$ will be important next. Finally, we introduce the minimum coverage of the optimal actions under $\pi_k$:
\begin{align*}
    C^\star_k = \inf_x \pi_k(a^\star(x)|x)\,.
\end{align*}

\begin{lemma} \label{lemma:acc}
    Let $\pi_k \in \Pi$, assume there exists $u \in [0, 1)$ such that $\Delta_u > 0$ and $C^\star_k > 0$, then:
    \begin{align*}
        L(\pi_\star, \pi_k) \le \gamma_k (R(\pi_k) - R(\pi_\star)),
    \end{align*}
with $\gamma_k = 1 + \left(\frac{1}{4} + \frac{1}{C_k^\star}\right)\frac{1}{\Delta_u (1 - u)} = \mathcal{O}\left(\frac{1}{(1 - u)\Delta_u C^\star_k} \right).$
\end{lemma}
The proof is given in \Cref{sec: proofs-sec-4}. The condition $\Delta_u > 0$ is analogous to margin assumptions commonly used in multi-arm contextual bandits \citep{fast_rate_cb, adaptive}, while $C^\star_k > 0$ imposes a coverage requirement ensuring that $\pi_k$ sufficiently explores optimal actions. Importantly, this condition is strictly weaker than the uniform coverage assumption imposed in \cite{adaptive, scrm}. Instead of controlling $L$, a similar inequality on the variance of importance weights was used as an assumption to obtain fast rates in \cite{scrm}. Here, we prove such result and characterize when it is possible to obtain. For instance, this inequality cannot hold when the costs are constant. Finally, while we formulate the conditions in terms of $\Delta_u$ and $C^\star_k$ for interpretability, a sharper, but less transparent, formulation is given in \Cref{app:convergence_lemma}.

\Cref{lemma:acc} shows that the expansion of $L$ is controlled by a difference of risks. Smaller $\gamma_k$ sharpens this inequality, which is the case for problems with distinguishable optimal actions and near optimal logging policies $\pi_k$. Plugging this inequality in \Cref{prop: suboptimality_with_variance} then yields a recursive control on the sequence $\left(R(\pi_{k})\right)_{k\geq 0}$ (see \Cref{cor: recursion}). Such a recursive relation can then be transformed into an explicit accelerated convergence rate for the output of \Cref{alg:s_analyzed}, provided the following assumption holds: for all $k \ge 0$, the algorithm outputs policies $\pi_k$ such that
\begin{align*}
    \forall k \ge 0\,, \quad C^\star_k \ge C^\star >0.
\end{align*}
Intuitively, this is plausible since the initial policy $\pi_0$ is expected to cover optimal actions, and subsequent policies $\pi_k$ for $k \ge 1$ should become progressively closer to $\pi_\star$. Formally, this condition is guaranteed by \Cref{assumption}, which requires the policy class to sufficiently cover the optimal actions. Combined with the separability condition ($\Delta_u > 0$ for some $u \in [0,1)$), we obtain the uniform upper bound on $\gamma_k$:
\begin{align*}
    \forall k, \quad \gamma_k \le \gamma = \frac{3}{\Delta_u(1 - u)C^\star} < +\infty\,.
\end{align*}
With this condition, we state our principal result below.



\begin{theorem}
    \label{th: suboptimality-conv-rate-unif-batch}
    Let $n_j=m>0$ for all $j$. Then, for any $\alpha \in[0,1)$, denote by $\pi_{k+1}$ the output of \Cref{alg:s_analyzed} at time $k$ with parameter $\lambda_m :=\frac{1-\alpha}{8\gamma\sqrt{m}}$. Then with probability at least $1-\delta$ over $S$,  for all $k\geq 0 $:
    \[ R(\pi_{k+1}) - R(\pi^\star) \leq \frac{C_{\alpha}}{\sqrt{m}(k + 1)^{\alpha}}, \]
    where $C_\alpha :=  64 \gamma \frac{\operatorname{KL}(\Q^\star,P) + \log(2/\delta)}{1 - \alpha}$.
\end{theorem}
\Cref{th: suboptimality-conv-rate-unif-batch} is proven in \Cref{sec: proofs-sec-4} and provides an accelerated rate of $\mathcal{O}\LP \frac{1}{\sqrt{m}k^\alpha}\RP$ compared to the $\mathcal{O}\LP \frac{1}{\sqrt{mk}}\RP$ of \Cref{prop: slowconvergence} as long as $\alpha > 1/2$. Again, we recover the same rate as \cite{sakhi2024logarithmic} in the batch setting (i.e. policy after one batch update: $k=1, m=n$) and improve the rate in the online setting (i.e. updating the behavior policy after each interaction: $k=n,m=1$) to $\mathcal{O}(1/n^\alpha)$. \Cref{th: suboptimality-conv-rate-unif-batch} also works for any intermediate regime between batch and online settings, and suggests that considering several updates is beneficial to learn efficiently. Indeed, as long as we choose a $\lambda$ that ensures $\alpha >1/2$, the number of updates $k$ will have a stronger influence than the size of the batch $m$.  It should be noted, however, that the constant $C_\alpha$ deteriorates as $\alpha$ approaches $1$. Larger values of $\lambda_m$ yield smaller constants but slower convergence rates, whereas smaller $\lambda_m$ yield faster rates at the cost of heavier constants. In practice, this creates a trade-off between accelerating convergence and keeping $C_\alpha$ at reasonable levels.

Finally, to capture a broad range of real-world scenarios where the size of the batches received $n_k$ may vary over time, we extend these results in \Cref{cor: suboptimality-conv-rate}, which provides convergence guarantees for \Cref{alg:s_analyzed} in this more general setting.

\begin{corollary}
    \label{cor: suboptimality-conv-rate}
    Let $\beta_1 := \sup_{k\geq0} \frac{k+1}{N_k} \leq 1$. Assume there exists $\beta_2 \in \mathbb{R}^+$ such that for all $k$, $ n_k \leq \beta_2$. For any $\alpha \in[0,1)$, denote by $\pi_{k+1}$ the output of \Cref{alg:s_analyzed} at time $k$ with parameter $\lambda =\frac{1}{1 + 2^{2+\alpha}\gamma \beta_1\beta_2 B(\alpha)}$. Then with probability at least $1-\delta$ over $S$, for all $k\geq 0 $:
    \[ R(\pi_{k+1}) - R(\pi^\star) \leq \frac{C_{\alpha}}{(k+1)^{\alpha}}, \]
    where $C_\alpha \leq \max \left( 1, \frac{\operatorname{KL}(\Q^\star,P) + \log(2/\delta)}{\lambda^2 \gamma \beta_2 B(\alpha)} \right)$, $B(\alpha)= \frac{1}{1-\alpha}$.
\end{corollary}



Observe that the convergence rate in \Cref{cor: suboptimality-conv-rate} is weaker than that of \Cref{th: suboptimality-conv-rate-unif-batch}, since it no longer depends explicitly on the batch sizes $n_j$. This relaxation is necessary to accommodate adversarial scenarios where some $n_j$ may be as small as 1. For example, let $n \geq 2$ and consider a sequence where $n_{2j-1} = m$ and $n_{2j} = 1$ for all $j \geq 1$. Then, for all $k \geq 0$,  we have $N_{2k-1} = k(m+1) -1$ and $N_{2k} = k(m+1)$. Thus, for all $k$, it holds that $\frac{k+1}{N_k} \leq \frac{2}{m}$, and the assumptions of \Cref{cor: suboptimality-conv-rate} are satisfied. Importantly, \Cref{cor: suboptimality-conv-rate} matches the optimal convergence rate of \Cref{th: suboptimality-conv-rate-unif-batch} in the fully online regime (i.e., $k=n$, $n_j=1$). Having established that \Cref{alg:s_analyzed} enjoys provably accelerated convergence rates, we turn to empirical validation to assess whether this theoretical advantage is reflected in practice.

\begin{figure*}[!ht]
    \centering
    \includegraphics[width=\textwidth]{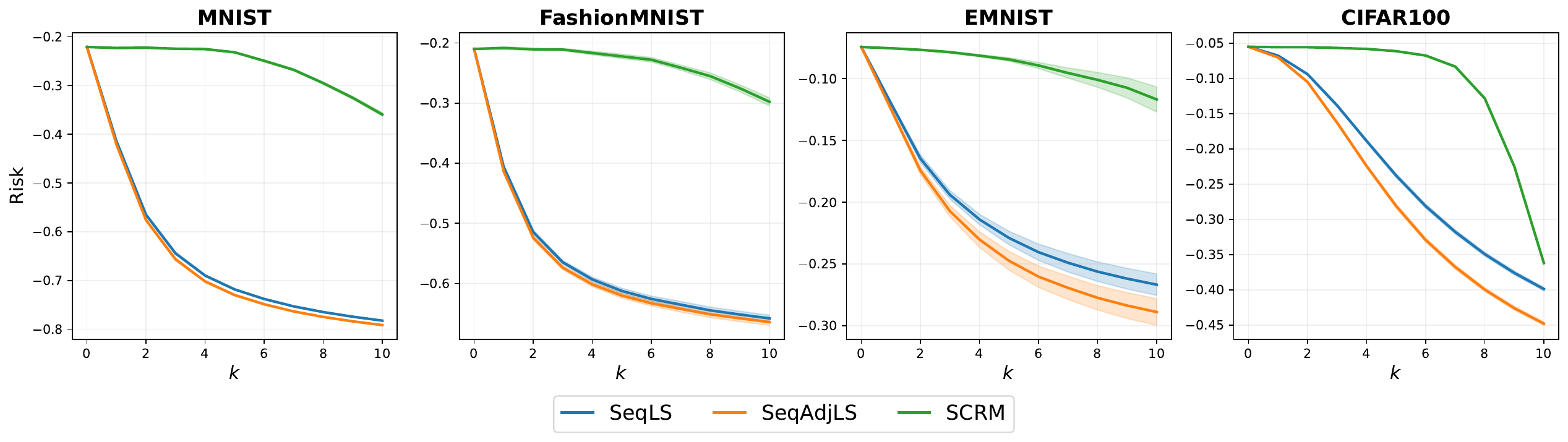}
    \caption{$R(\pi_k)$ at each intermediate step, for $k = 10$.}
    \label{fig:your_label}
\end{figure*}

\section{EXPERIMENTS}
\label{sec: expes}
\textbf{General Setup.} Our experimental framework follows the standard multiclass-to-bandit conversion widely used in prior work \citep{dudik14doubly, swaminathan2015batch}. Each multiclass dataset consists of feature-label pairs, which we transform into contextual bandit problems by treating features as contexts and labels as actions. The reward $r$ for taking action $a$ for context $x$ is modeled as Bernoulli with probability $p_x = \epsilon + \mathds{1}\left[a = a^\star(x)\right](1 - 2\epsilon)$, where $a^\star(x)$ is the true label of features $x$, and $\epsilon$ is a noise parameter, that we set to $\epsilon = 0.2$ in all our experiments. This setup ensures an average reward of $1 - \epsilon$ for the optimal action $a^\star(x)$ and $\epsilon$ for all others. In our experiments, we use parametrized Linear Gaussian Policies \citep{sakhi2023pac, aouali23a} as our class of policies. These policies interact with the contextual bandit environment to construct a logged bandit feedback dataset in the form $\{x_i, a_i, c_i\}_{i \in [n]}$, where $c_i = - r_i$ is the associated cost. \Cref{app:experiments} details the experimental design and obtained results.

\textbf{Effect of Multiple Deployments.} Unlike prior batch approaches, our sequential algorithms support multiple updates of the behavior policy $\pi_0$, while seamlessly reducing to the batch case when $k = 1$. Using the \texttt{MNIST} \citep{mnist} and \texttt{CIFAR100} \citep{cifar100} datasets, we evaluate how the number of intermediate policy updates $k$ affects the performance of the final learned policy $\pi$. Starting from a uniform policy $\pi_0$, we run both \texttt{SeqLS} (\texttt{Alg}~\ref{alg:extension-sakhi}) and the improved \texttt{SeqAdjLS} (\texttt{Alg}~\ref{alg:s_analyzed}). For each $k \in \{1, 5, 10, 100\}$, we fix the total number of interactions to $N_k = N$, resulting in $m = N/k$ interactions per deployment. The case $k = 1$ corresponds to the static setting studied in \cite{sakhi2024logarithmic}, where all data is collected using $\pi_0$. For $k > 1$, each intermediate policy $\pi_j$ interacts with $m$ examples, collectively producing $N$ sequentially collected samples. We set $\lambda = 1/\sqrt{m}$, following our theoretical guidance. 

\Cref{tab:k_effect} summarizes the results: increasing $k$ consistently improves final policy performance, with diminishing returns as $\pi$ nears optimality (e.g., \texttt{MNIST} at $k = 10$ and $100$). The gains are especially notable on \texttt{CIFAR100}, where the larger action space ($|\mathcal{A}| = 100$) makes uniform policies particularly suboptimal, thus highlighting the importance of the sequential updates. Finally, our results demonstrate the empirical advantage of \texttt{SeqAdjLS}, validating the need for our refined estimator. Across all settings, and especially the more challenging ones, \texttt{SeqAdjLS} consistently outperforms \texttt{SeqLS}, providing empirical evidence that the adjustment is essential for achieving accelerated convergence.

\begin{table}
    \centering
\caption{$R(\pi)$ while fixing $N_k = N$, varying $k$ in \texttt{MNIST} and \texttt{CIFAR100} (lower is better).}
    \label{tab:k_effect}
    \begin{tabular}{c|c|c||ccc}
        \toprule
         Dataset & $k$ & \texttt{SeqLS} (\ref{alg:extension-sakhi}) & \texttt{SeqAdjLS} (\ref{alg:s_analyzed}) \\
        \midrule
        \multirow{4}{*}{\texttt{MNIST}}&  
         $1$ & $-0.674$ & $-0.678$  \\
        
         &  $5$ & $-0.758$ & $-0.766$   \\
        
         &  $10$ & $-0.782$ & $\mathbf{-0.791}$   \\
         
         &  $100$ & $-0.781$ & $\mathbf{-0.789}$    \\

         \cmidrule{1-4}

         \multirow{4}{*}{\texttt{CIFAR100}} &  $1$ & $-0.197$ & $-0.207$  \\

         & $5$ & $-0.344$ & $-0.380$ \\
        
         & $10$ & $-0.406$ & $-0.455$  \\

         & $100$ & $-0.508$ & $\mathbf{-0.558}$  \\

        \bottomrule
    \end{tabular}
    \vspace{-0.1cm}
\end{table}

\textbf{Benchmarks.} Having demonstrated the benefits of intermediate updates and the effectiveness of our proposed \texttt{SeqAdjLS} algorithm, we now compare our methods against the recently introduced \texttt{SCRM} approach \cite{scrm}, which extends the \texttt{CRM} principle of \cite{swaminathan2015batch} to the sequential setting. We evaluate all methods on \texttt{MNIST}, \texttt{FashionMNIST} \citep{xiao2017fashion}, \texttt{EMNIST} \citep{cohen2017emnist}, and \texttt{CIFAR100} with varying action space sizes, fixing the total number of interactions to $N$, setting $k = 10$, and using theoretically recommended values for $\lambda$ to ensure a fair comparison.

\texttt{SeqLS} and \texttt{SeqAdjLS} divide the interaction budget uniformly across rounds ($n_j = N/k$), reusing all previously collected data to update policies. In contrast, \texttt{SCRM} employs exponentially increasing interaction sizes ($n_j = n_0 \cdot 2^j$ with $n_0 = \lceil N/2^k \rceil$) and updates policies using only the most recent batch, rather than the cumulative data ($n_j$ instead of $N_j$). Each method is evaluated across six random seeds, and the results are summarized in \Cref{fig:your_label}.

We observe that \texttt{SCRM} tends to be overly conservative, particularly in early rounds when the number of interactions is small. Although it improves upon the initial behavior policy $\pi_0$ over time, it lags significantly behind both of our sequential methods. While it narrows the gap with \texttt{SeqLS} on \texttt{CIFAR100}, it is consistently outperformed, especially by \texttt{SeqAdjLS}, which achieves superior results across all datasets. These findings further confirm the empirical advantages of our approach. Additional experimental details and results are provided in \Cref{app:experiments}.

\section{CONCLUSION}
\label{sec: ccl}

In this work, we extended the Logarithmic Smoothing learning principle to the sequential setting. We proposed two novel sequential algorithms for off-policy learning and established convergence guarantees to their respective optima. One of these algorithms enjoys accelerated convergence under mild assumptions and demonstrated improved empirical performance. These theoretical and practical benefits highlight the potential of combining online PAC-Bayes techniques with sequential contextual bandits. Our theoretical analysis requires our policy class to cover the optimal actions, thus weakening standard uniform coverage assumptions. An interesting open question is whether the structure of the pessimistic update can be exploited to relax this requirement. Moreover, in \Cref{sec: non-iid-contexts}, we showed that online PAC-Bayes techniques yield learning algorithms that relax the \emph{i.i.d.} assumption on contexts, thereby addressing real-world scenarios with distribution shifts. A natural and promising direction is to investigate whether the (accelerated) convergence guarantees established here extend to this more general, non-stationary setting. 
\paragraph{Acknowledgements}
M.H. is partially supported by the European Research Council Starting Grant DYNASTY – 101039676. Views and opinions expressed are however those of the authors only and do
not necessarily reflect those of the European Union or
the European Research Council Executive Agency. Neither the European Union nor the granting authority
can be held responsible for them.
\bibliography{biblio}

@article{xiao2017fashion,
  title={Fashion-mnist: a novel image dataset for benchmarking machine learning algorithms},
  author={Xiao, Han and Rasul, Kashif and Vollgraf, Roland},
  journal={arXiv preprint arXiv:1708.07747},
  year={2017}
}

@techreport{ cifar100,
  author = "Alex Krizhevsky",
  title = "Learning Multiple Layers of Features from Tiny Images",
  institution = "University of Toronto",
  year = "2009"
}

@inproceedings{cohen2017emnist,
  title={EMNIST: Extending MNIST to handwritten letters},
  author={Cohen, Gregory and Afshar, Saeed and Tapson, Jonathan and Van Schaik, Andre},
  booktitle={2017 international joint conference on neural networks (IJCNN)},
  pages={2921--2926},
  year={2017},
  organization={IEEE}
}

@misc{ mnist,
  author = "Yann LeCun and Corinna Cortes and Christopher Burges",
  title = "{MNIST Handwritten Digit Database}",
  howpublished = "\url{http://yann.lecun.com/exdb/mnist}",
  year = "2010"
}

@article{alquier2024user,
  title={{User-friendly Introduction to {PAC-Bayes} Bounds}},
  author={Alquier, Pierre},
  journal={Foundations and Trends{\textregistered} in Machine Learning},
  year={2024},
  publisher={Now Publishers, Inc.}
}

@inproceedings{andreeva2024topological,
  author       = {Rayna Andreeva and
                  Benjamin Dupuis and
                  Rik Sarkar and
                  Tolga Birdal and
                  Umut {\c S}im{\c s}ekli},
  title        = {{Topological Generalization Bounds for Discrete-Time Stochastic Optimization
                  Algorithms}},
  booktitle    = {Advances in Neural Information Processing Systems (NeurIPS)},
  year         = {2024},
}

@article{bottou2013counterfactual,
  title={Counterfactual Reasoning and Learning Systems: The Example of Computational Advertising.},
  author={Bottou, L{\'e}on and Peters, Jonas and Qui{\~n}onero-Candela, Joaquin and Charles, Denis X and Chickering, D Max and Portugaly, Elon and Ray, Dipankar and Simard, Patrice and Snelson, Ed},
  journal={Journal of Machine Learning Research},
  volume={14},
  number={11},
  year={2013}
}

@inproceedings{adaptive,
author = {Bibaut, Aur\'{e}lien and Kallus, Nathan and Dimakopoulou, Maria and Chambaz, Antoine and van der Laan, Mark},
title = {Risk minimization from adaptively collected data: guarantees for supervised and policy learning},
year = {2021},
isbn = {9781713845393},
publisher = {Curran Associates Inc.},
address = {Red Hook, NY, USA},
booktitle = {Proceedings of the 35th International Conference on Neural Information Processing Systems},
articleno = {1473},
numpages = {13},
series = {NIPS '21}
}

@inproceedings{kallus2018policy,
  title={Policy evaluation and optimization with continuous treatments},
  author={Kallus, Nathan and Zhou, Angela},
  booktitle={International conference on artificial intelligence and statistics},
  pages={1243--1251},
  year={2018},
  organization={PMLR}
}

@article{swaminathan2015batch,
  title={Batch learning from logged bandit feedback through counterfactual risk minimization},
  author={Swaminathan, Adith and Joachims, Thorsten},
  journal={The Journal of Machine Learning Research},
  volume={16},
  number={1},
  pages={1731--1755},
  year={2015},
  publisher={JMLR. org}
}

@inproceedings{london2019bayesian,
  title={Bayesian counterfactual risk minimization},
  author={London, Ben and Sandler, Ted},
  booktitle={International Conference on Machine Learning},
  pages={4125--4133},
  year={2019},
  organization={PMLR}
}

@InProceedings{aouali23a,
  title = 	 {{Exponential Smoothing for Off-Policy Learning}},
  author =       {Aouali, Imad and Brunel, Victor-Emmanuel and Rohde, David and Korba, Anna},
  booktitle = 	 {Proceedings of the 40th International Conference on Machine Learning},
  pages = 	 {984--1017},
  year = 	 {2023},
  publisher =    {PMLR},
}

@article{catoni2003pac,
  title={{A PAC-Bayesian approach to adaptive classification}},
  author={Catoni, Olivier},
  journal={preprint},
  volume={840},
  year={2003}
  }

@book{catoni2007pac,
  author    = {Olivier Catoni},
  title     = {{PAC-Bayesian supervised classification: the thermodynamics of statistical learning}},
  publisher = {Institute of Mathematical Statistics},
  year      = {2007},
}

@article{chugg2023unified,
  author  = {Ben Chugg and Hongjian Wang and Aaditya Ramdas},
  title   = {{A Unified Recipe for Deriving (Time-Uniform) PAC-Bayes Bounds}},
  journal = {Journal of Machine Learning Research},
  year    = {2023},
  url     = {http://jmlr.org/papers/v24/23-0401.html}
}

@article{csizar1975divergence,
  author       = {I. Csisz{\'a}r},
  title        = {{$I$-Divergence Geometry of Probability Distributions and Minimization Problems}},
  journal      = {The Annals of Probability},
  year         = {1975},
}

@article{donsker1976asymptotic,
  author       = {M. D. Donsker and
                  S. R. S. Varadhan},
  title        = {{Asymptotic evaluation of certain Markov process expectations for large time---III}},
  journal      = {Communications on Pure and Applied Mathematics},
  year         = {1976},
}

@article{doob1939jean,
  title={{Jean Ville, {\'E}tude Critique de la Notion de Collectif}},
  author={Doob, JL},
  journal={Bulletin of the American mathematical society},
  volume={45},
  number={11},
  pages={824--824},
  year={1939},
  publisher={American Mathematical Society}
}

@inproceedings{dziugaite2017computing,
  author    = {Gintare Karolina Dziugaite and
               Daniel Roy},
  title     = {{Computing Nonvacuous Generalization Bounds for Deep (Stochastic) Neural Networks with Many More Parameters than Training Data}},
  booktitle = {Conference on Uncertainty in Artificial Intelligence (UAI)},
  year      = {2017},
}

@inproceedings{fard2010pac,
  author       = {Mahdi Milani Fard and
                  Joelle Pineau},
  title        = {{PAC-Bayesian} Model Selection for Reinforcement Learning},
  booktitle    = {Conference on Neural Information Processing Systems (NeurIPS)},
  year         = {2010},
}

@InProceedings{gabbianelli2023importance,
  title = 	 {{Importance-Weighted Offline Learning Done Right}},
  author =       {Gabbianelli, Germano and Neu, Gergely and Papini, Matteo},
  booktitle = 	 {Proceedings of The 35th International Conference on Algorithmic Learning Theory},
  pages = 	 {614--634},
  year = 	 {2024},
  volume = 	 {237},
  series = 	 {Proceedings of Machine Learning Research},
  month = 	 {25--28 Feb},
  publisher =    {PMLR},
}

@inproceedings{guedj2019primer,
  author    = {Benjamin Guedj},
  title     = {{A Primer on {PAC-Bayesian} Learning}},
  booktitle = {Proceedings of the second congress of the French Mathematical Society},
  year      = {2019},
}

@inproceedings{haddouche2022online,
  author    = {Maxime Haddouche and
               Benjamin Guedj},
  title     = {{Online PAC-Bayes Learning}},
  booktitle = {Advances in Neural Information Processing Systems (NeurIPS)},
  year      = {2022}
}

@article{haddouche2023pac,
  author    = {Maxime Haddouche and
               Benjamin Guedj},
  title     = {{PAC-Bayes Generalisation Bounds for Heavy-Tailed Losses through Supermartingales}},
  journal   = {Transactions on Machine Learning Research},
  year      = {2023}
}

@book{smc_book,
	address = {Cham, Switzerland},
	author = {Chopin, Nicolas and Papaspiliopoulos, Omiros},
	booktitle = {An introduction to Sequential Monte Carlo},
	edition = {1st ed. 2020.},
	isbn = {3-030-47845-9},
	language = {eng},
	publisher = {Springer},
	series = {Springer Series in Statistics},
	title = {An introduction to Sequential Monte Carlo / Nicolas Chopin, Omiros Papaspiliopoulos.},
	year = {2020}}

@article{fast_rate_cb,
author = {Hu, Yichun and Kallus, Nathan and Mao, Xiaojie},
title = {Fast Rates for Contextual Linear Optimization},
year = {2022},
issue_date = {June 2022},
publisher = {INFORMS},
address = {Linthicum, MD, USA},
volume = {68},
number = {6},
issn = {0025-1909},
url = {https://doi.org/10.1287/mnsc.2022.4383},
doi = {10.1287/mnsc.2022.4383},
journal = {Manage. Sci.},
month = jun,
pages = {4236–4245},
numpages = {10}
}

@inproceedings{jin2021pessimism,
  title={Is pessimism provably efficient for offline rl?},
  author={Jin, Ying and Yang, Zhuoran and Wang, Zhaoran},
  booktitle={International Conference on Machine Learning},
  pages={5084--5096},
  year={2021},
  organization={PMLR}
}

@inproceedings{mcallester1998some,
  Author = {McAllester, David A},
  Booktitle = {Proceedings of the eleventh annual conference on Computational Learning Theory},
  Organization = {ACM},
  Pages = {230--234},
  Title = {{Some PAC-Bayesian theorems}},
  Year = {1998}}

@inproceedings{mcallester1999pac,
  Author = {McAllester, David A},
  Booktitle = {Proceedings of the twelfth annual conference on Computational Learning Theory},
  Organization = {ACM},
  Pages = {164--170},
  Title = {{PAC-Bayesian model averaging}},
  Year = {1999}}

@article{mcallester2003pac,
  author       = {McAllester, David A},
  title        = {{PAC-Bayesian Stochastic Model Selection}},
  journal      = {Machine Learning},
  year         = {2003},
}

@article{metelli2021subgaussian,
  title={Subgaussian and differentiable importance sampling for off-policy evaluation and learning},
  author={Metelli, Alberto Maria and Russo, Alessio and Restelli, Marcello},
  journal={Advances in Neural Information Processing Systems},
  volume={34},
  pages={8119--8132},
  year={2021}
}

@inproceedings{mou2018generalization,
  author       = {Wenlong Mou and
                  Liwei Wang and
                  Xiyu Zhai and
                  Kai Zheng},
  title        = {{Generalization Bounds of SGLD for Non-convex Learning: Two Theoretical Viewpoints}},
  booktitle    = {Conference On Learning Theory (COLT)},
  year         = {2018},
}

@inproceedings{nozawa2019pacbayesian,
title={{PAC-Bayesian} Contrastive Unsupervised Representation Learning},
author={Kento Nozawa and Pascal Germain and Benjamin Guedj},
year={2020},
booktitle = {{Conference on Uncertainty in Artificial Intelligence [UAI]}},
}

@article{perez2021tighter,
  author    = {Maria Perez-Ortiz and
               Omar Rivasplata and
               John Shawe-Taylor and
               Csaba Szepesvari},
  title     = {{Tighter Risk Certificates for Neural Networks}},
  journal   = {Journal of Machine Learning Research},
  year      = {2021},
}

@inproceedings{sakhi2023pac,
  title={{PAC-Bayesian Offline Contextual Bandits with Guarantees}},
  author={Sakhi, Otmane and Alquier, Pierre and Chopin, Nicolas},
  booktitle={International Conference on Machine Learning},
  pages={29777--29799},
  year={2023},
  organization={PMLR}
}

@article{Sakhi_Rohde_Gilotte_2023,
title={{Fast Offline Policy Optimization for Large Scale Recommendation}}, volume={37}, number={8}, journal={Proceedings of the AAAI Conference on Artificial Intelligence}, author={Sakhi, Otmane and Rohde, David and Gilotte, Alexandre}, year={2023}, month={Jun.}, pages={9686-9694} }

@inproceedings{sakhi2024logarithmic,
title={{Logarithmic Smoothing for Pessimistic Off-Policy Evaluation, Selection and Learning}},
author={Otmane Sakhi and Imad Aouali and Pierre Alquier and Nicolas Chopin},
booktitle={The Thirty-eighth Annual Conference on Neural Information Processing Systems},
year={2024},
}

@article{ dudik14doubly,
  author = "Miroslav Dudik and Dumitru Erhan and John Langford and Lihong Li",
  title = "Doubly Robust Policy Evaluation and Optimization",
  journal = "Statistical Science",
  volume = "29",
  number = "4",
  pages = "485-511",
  year = "2014"
}

@book{mcbook,
   author = {Art B. Owen},
   year = 2013,
   title = {Monte Carlo theory, methods and examples},
   publisher = {\url{https://artowen.su.domains/mc/}}
}

@inproceedings{optimistic_shrinkage,
  title={Doubly robust off-policy evaluation with shrinkage},
  author={Su, Yi and Dimakopoulou, Maria and Krishnamurthy, Akshay and Dud{\'\i}k, Miroslav},
  booktitle={International Conference on Machine Learning},
  pages={9167--9176},
  year={2020},
  organization={PMLR}
}

@article{horvitz1952generalization,
  title={A generalization of sampling without replacement from a finite universe},
  author={Horvitz, Daniel G and Thompson, Donovan J},
  journal={Journal of the American statistical Association},
  volume={47},
  number={260},
  pages={663--685},
  year={1952},
  publisher={Taylor \& Francis}
}

@misc{zhan2022policylearningadaptivelycollected,
      title={Policy Learning with Adaptively Collected Data}, 
      author={Ruohan Zhan and Zhimei Ren and Susan Athey and Zhengyuan Zhou},
      year={2022},
      eprint={2105.02344},
      archivePrefix={arXiv},
      primaryClass={stat.ML},
      url={https://arxiv.org/abs/2105.02344}, 
}

@article{kingma2014adam,
  title={Adam: A method for stochastic optimization},
  author={Kingma, Diederik P and Ba, Jimmy},
  journal={arXiv preprint arXiv:1412.6980},
  year={2014}
}

@inproceedings{blob,
author = {Sakhi, Otmane and Bonner, Stephen and Rohde, David and Vasile, Flavian},
title = {{BLOB}: A Probabilistic Model for Recommendation that Combines Organic and Bandit Signals},
year = {2020},
isbn = {9781450379984},
publisher = {Association for Computing Machinery},
address = {New York, NY, USA},
doi = {10.1145/3394486.3403121},
booktitle = {Proceedings of the 26th ACM SIGKDD International Conference on Knowledge Discovery \& Data Mining},
pages = {783–793},
numpages = {11},
location = {Virtual Event, CA, USA},
series = {KDD '20}
}

@inproceedings{pac_bayes_unified,
author = {Aouali, Imad and Brunel, Victor-Emmanuel and Rohde, David and Korba, Anna},
title = {Unified PAC-Bayesian study of pessimism for offline policy learning with regularized importance sampling},
year = {2024},
publisher = {JMLR.org},
booktitle = {Proceedings of the Fortieth Conference on Uncertainty in Artificial Intelligence},
articleno = {4},
numpages = {22},
location = {Barcelona, Spain},
series = {UAI '24}
}

@inproceedings{scrm,
author = {Zenati, Houssam and Diemert, Eustache and Martin, Matthieu and Mairal, Julien and Gaillard, Pierre},
title = {Sequential counterfactual risk minimization},
year = {2023},
publisher = {JMLR.org},
booktitle = {Proceedings of the 40th International Conference on Machine Learning},
articleno = {1705},
numpages = {26},
location = {Honolulu, Hawaii, USA},
series = {ICML'23}
}

@inproceedings{shawe1997pac,
  Author = {J. Shawe-Taylor and R. C. Williamson},
  Booktitle = {Proceedings of the 10th annual conference on Computational Learning Theory},
  Organization = {ACM},
  Pages = {2--9},
  Title = {{A PAC analysis of a Bayes estimator}},
  Year = {1997}}
\bibliographystyle{abbrvnat}

\newpage

\section*{Checklist}

\begin{enumerate}

  \item For all models and algorithms presented, check if you include:
  \begin{enumerate}
    \item A clear description of the mathematical setting, assumptions, algorithm, and/or model. [Yes] We provide all proofs and algorithm statement either in the main text or appendices.
    \item An analysis of the properties and complexity (time, space, sample size) of any algorithm. [Yes] We study convergence guarantees of our algorithms as well as providing PAC-Bayes bound which justifies their existence.
    \item (Optional) Anonymized source code, with specification of all dependencies, including external libraries. [Yes]
  \end{enumerate}

  \item For any theoretical claim, check if you include:
  \begin{enumerate}
    \item Statements of the full set of assumptions of all theoretical results. [Yes] All results are stated with respect to our detailed theoretical setup of \Cref{sec: framework}.
    \item Complete proofs of all theoretical results. [Yes] All proofs are detailed in the appendix due to space constraints.
    \item Clear explanations of any assumptions. [Yes] We explain and contextualize our various assumptions in \Cref{sec: framework,sec: warmup-pls,sec: accelerated-convergence}.    
  \end{enumerate}

  \item For all figures and tables that present empirical results, check if you include:
  \begin{enumerate}
    \item The code, data, and instructions needed to reproduce the main experimental results (either in the supplemental material or as a URL). [Yes] Code lies in the supplementary material.
    \item All the training details (e.g., data splits, hyperparameters, how they were chosen). [Yes] All practical details are gathered in \Cref{sec: expes,app:experiments}.
    \item A clear definition of the specific measure or statistics and error bars (e.g., with respect to the random seed after running experiments multiple times). [Yes] Such details are retrievable in \Cref{sec: expes,app:experiments}.
    \item A description of the computing infrastructure used. (e.g., type of GPUs, internal cluster, or cloud provider). [Yes] All experiments were run on a single machine. Such details are retrievable in \Cref{app:experiments}. 
  \end{enumerate}

  \item If you are using existing assets (e.g., code, data, models) or curating/releasing new assets, check if you include:
  \begin{enumerate}
    \item Citations of the creator If your work uses existing assets. [Yes] We mention all the inspirations we exploited as well as references of the datasets we used in \Cref{sec: expes}.
    \item The license information of the assets, if applicable. [Not Applicable]
    \item New assets either in the supplemental material or as a URL, if applicable. [Not Applicable]
    \item Information about consent from data providers/curators. [Not Applicable] We only exploited standard datasets such as MNIST and CIFAR100, publicly available and not requiring any specific consent to be used.
    \item Discussion of sensible content if applicable, e.g., personally identifiable information or offensive content. [Not Applicable]
  \end{enumerate}

  \item If you used crowdsourcing or conducted research with human subjects, check if you include:
  \begin{enumerate}
    \item The full text of instructions given to participants and screenshots. [Not Applicable]
    \item Descriptions of potential participant risks, with links to Institutional Review Board (IRB) approvals if applicable. [Not Applicable]
    \item The estimated hourly wage paid to participants and the total amount spent on participant compensation. [Not Applicable]
  \end{enumerate}

\end{enumerate}

\clearpage
\appendix
\thispagestyle{empty}

\onecolumn
\aistatstitle{APPENDIX}

\section{EXTENDED RELATED WORK}

\paragraph{Off-Policy Learning.} Off-Policy Learning in the contextual bandit setting, or counterfactual risk minimization \citep{swaminathan2015batch} is a widely adopted framework for learning from past interactions. In most real world applications, large datasets summarizing past interactions are often available, allowing agents to improve their policies offline. In this setting, the goal is to find the optimal policy using previously collected interactions.  Two primary approaches exist for learning from this data. The direct method \citep{blob} focuses on modeling the reward and subsequently derives a straightforward policy. This approach is well-understood, and the quality of the learned policy depends entirely on our ability to accurately model the reward. The other approach directly learns a policy from the data by minimizing an importance weighting estimator \citep{dudik14doubly, bottou2013counterfactual} of the expected cost. The research community has predominantly focused on this second approach, as it presents significant theoretical challenges and offers greater generality by not relying on any reward modeling. At the core of directly learning a policy is the development of tailored off-policy estimators with improved concentration properties \citep{bottou2013counterfactual, aouali23a, metelli2021subgaussian, sakhi2024logarithmic}, combined with tight generalization bounds allowing efficient learning \citep{swaminathan2015batch, sakhi2023pac, sakhi2024logarithmic}. The used methods are motivated by the pessimism principle \citep{jin2021pessimism} that was proven to be optimal for learning in these scenarios. Particularly, analyzing this learning problem from the PAC-Bayesian learning theory lens \cite{london2019bayesian} proved efficient, improving the seminal work of \citep{swaminathan2015batch}. This direction was further explored by \cite{sakhi2023pac, aouali23a, gabbianelli2023importance, sakhi2024logarithmic}, reducing assumptions and improving the pessimistic principle, with latest developments \citep{sakhi2024logarithmic} resulting in state of the art pessimistic approaches. We further develop this research direction and generalize it to the sequential setting, where an online algorithm redeploys the learned policy to the environment and collect more interactions to train the next iteration.

\paragraph{Sequential Off-Policy Learning.} Sequential policy learning extends off-policy learning by allowing the decision-maker to deploy a learned policy, collect additional interactions, and then train the next policy on the entire set of previously gathered data. This approach is widely used in practice since decision systems are continuously updated. Nevertheless, its theoretical foundations remain limited and underexplored, and only a few works have addressed this question. 

A related direction is off-policy learning from adaptively collected data, studied in \cite{zhan2022policylearningadaptivelycollected} and \cite{adaptive}. These works focus on purely offline settings, where data is collected under multiple logging policies with decaying exploration. However, they remain one-shot batch approaches: their goal is to learn a single policy from the collected interactions, without analyzing the effects of redeployment. In contrast, we consider sequential updates and redeployments that occur after the arrival of batches of arbitrary sizes. Our analysis characterizes the conditions under which sequential deployment achieves fast convergence rates. This framework more accurately models real-world scenarios, provides theoretical insight into a widely used practical approach, and unifies the static (fully batch) and dynamic (fully online) regimes. Our proposed algorithms match the convergence rates of the batch setting and further accelerate learning as the behavior policy becomes refined.

Recently, \cite{scrm} extended \emph{Counterfactual Risk Minimization} \citep{swaminathan2015batch} to the sequential setting, allowing deployments of learned policies within the same framework as ours. Their algorithm achieves accelerated convergence under a Hölderian assumption, but this assumption is not always justified. Moreover, it requires bounded importance weights, which in discrete action spaces is equivalent to assuming uniform coverage. Their method further relies on unrealistic, exponentially growing batch sizes and only uses the most recent batch of data, discarding all past interactions. These limitations arise from the theoretical tools employed, making the algorithm impractical. By contrast, we develop a PAC-Bayesian approach that avoids these shortcomings. Our algorithms are the first to seamlessly interpolate between full batch and full online regimes, achieving the optimal convergence rates of both while requiring only coverage of the optimal actions.

\paragraph{PAC-Bayes learning.} PAC-Bayes learning is a recent branch of learning theory which emerged in the late 90s via the seminal work of \citep{shawe1997pac,mcallester1998some,mcallester1999pac,mcallester2003pac} and later pursued by \citep{catoni2003pac,catoni2007pac}. Modern surveys are available to describe the various advances in the field \citep{guedj2019primer,alquier2024user}. PAC-Bayes provide generalization bounds involving a complexity term (often a KL divergence), inspired here from the Bayesian learning paradigm of designing a \emph{posterior} knowledge of the learning problem based on both training data and a \emph{prior} knowledge of the considered situation. On the contrary, PAC-Bayes, while inspired from the Bayesian philosophy, relies historically on tools from information theory. This general approach benefits from additional flexibility as PAC-Bayes can be linked and applied to Bayesian learning (see \citealp{guedj2019primer}) but also blurs the notion of prior and posterior distributions, now independent of the fundamental Bayes formula. This flexibility allowed PAC-Bayes to reach various sub-fields of learning theory: optimization dynamics of learning algorithms \citep{mou2018generalization, andreeva2024topological}, reinforcement learning \citep{fard2010pac}, online learning \citep{haddouche2022online} and contrastive learning \citep{nozawa2019pacbayesian} to name a few.

\section{PROOFS OF SECTION \ref{sec: warmup-pls}}
\label{sec: extensions-section-3}
In this whole section, we assume that all datasets $S_k$ have the same cardinal $n_k=m$.
This section contains the proof of \Cref{th: extension-sakhi} and \Cref{prop: slowconvergence}. 

\subsection{Proof of Theorem \ref{th: extension-sakhi}}

We start by restating the theorem of interest: 

\extensionsakhi*

\begin{proof}[Proof of \Cref{th: extension-sakhi}]
    For any $k\geq 0$, let:
\begin{align*}
   f_k(S,\theta) := \sum_{j=0}^k \underbrace{\sum_{i\in I_k} -\log\LP 1-\lambda R(d_{\theta}) \RP + \log\LP 1 - \lambda\frac{d_{\theta}(a_i\mid x_i) c_i}{\pi_{j}(a_i\mid x_i)} \RP}_{:= Y_j(S,\theta)}.
   \end{align*}

   Now for a given posterior $Q$ and a prior $P$, we properly apply the change of measure inequality for any $m$:
   \begin{align}
   \label{eq: change-meas-opl-subbatch-sec3}
    \mathbb{E}_{\theta\sim Q}\left[ f_k(S,\theta) \right] 
     \leq \operatorname{KL}(Q,P) + \log \left( \mathbb{E}_{\theta\sim P}\exp(f_k(S,\theta))  \right).
   \end{align}

   We need to control the exponential random variable. 
   To do so we prove the following lemma:

   \begin{lemma}
     The sequence $(M_k:=\mathbb{E}_{\theta\sim P}\exp(f_k(S,\theta))_{k\geq 0}$ is a non-negative martingale.
   \end{lemma}
   \begin{proof}
   We fix $k\geq 0$. We show that $\mathbb{E}_{k-1}[M_k] = M_{k-1}$.
     \begin{align*}
       \mathbb{E}_{k-1}[M_k] &=\mathbb{E}_{k-1}\left[\mathbb{E}_{\theta \sim P}\exp(f_k(S,\theta)\right] \\
        &= \mathbb{E}_{\theta \sim P}\left[\mathbb{E}_{k-1}\left[\exp\LP f_{k-1}(S,\theta) \RP \exp\left(\sum_{i\in I_k}\log\LP \frac{1 - \lambda\frac{d_{\theta}(a_i\mid x_i) c_i}{\pi_{k}(a_i\mid x_i)}}{1-\lambda R(d_{\theta})}  \RP \right)\right] \right] \\
        &= \mathbb{E}_{\theta \sim P}\left[\exp\LP f_{k-1}(S,\theta) \RP \mathbb{E}_{k-1}\left[ \exp\left(\sum_{i\in I_k}\log\LP \frac{1 - \lambda\frac{d_{\theta}(a_i\mid x_i) c_i}{\pi_{k}(a_i\mid x_i)}}{1-\lambda R(d_{\theta})}  \RP \right)\right] \right].
     \end{align*}
 The second line comes from the application of Fubini as $P$ is data-independent and last line already demonstrated.
  Now we have for any $\theta\in\Theta$:

   \begin{align*}
     \mathbb{E}_{k-1} \left[ \exp\left(\sum_{i\in I_k}\log\LP \frac{1 - \lambda\frac{d_{\theta}(a_i\mid x_i) c_i}{\pi_{k}(a_i\mid x_i)}}{1-\lambda R(d_{\theta})}  \RP \right)\right]  & =  \mathbb{E}_{k-1}\left[\Pi_{i\in I_k}\frac{1 - \lambda\frac{d_{\theta}(a_i\mid x_i) c_i}{\pi_{k}(a_i\mid x_i)}}{1-\lambda R(d_{\theta})} \right]\\
     & = \Pi_{i\in I_k}\frac{1 - \lambda\Ebb_{k-1}\LB\frac{d_{\theta}(a_i\mid x_i) c_i}{\pi_{k}(a_i\mid x_i)}\RB}{1-\lambda R(d_{\theta})},
   \end{align*}
these computations holding as we exploited the assumption that, conditionally to $\mathcal{F}_{k-1}$, $S_k= \LP (x_i,a_i,c_i) \RP_{i\in I_k}$ is constituted of mutually independent random variables (thus an expectation of product equals a product of expectations). 
 Now, recall that $\pi_{k}$ is $\mathcal{F}_{k-1}$-measurable and that for all $i\in I_k, a_i \sim \pi_k(\cdot\mid x_i)$. We then remark that, for any $\theta$, $i\in I_k$:
    \begin{align*}
        \mathbb{E}_{k-1}\left[\frac{d_{\theta}(a_i\mid x_i) c_i}{\pi_{k}(a_i\mid x_i)}\right] & = \Ebb_{x'_i\sim \nu} \Ebb_{a'_i \sim \pi_{k}(\cdot \mid x'_i)} \Ebb_{c'_i\sim p(\cdot\mid x'_i,a'_i)} \LB \frac{d_{\theta}(a'_i\mid x'_i) c'_i}{\pi_{k}(a'_i\mid x'_i)} \RB \\
        & = \Ebb_{x'_i\sim \nu} \Ebb_{a'_i \sim \pi_{k}(\cdot \mid x'_i)} \LB \frac{d_{\theta}(a'_i\mid x'_i)}{\pi_{k}(a'_i\mid x'_i)} \Ebb_{c'_i\sim p(\cdot\mid x'_i,a'_i)} \LP c'_i\RP\RB \\
        & = \Ebb_{x'_i\sim \nu} \Ebb_{a'_i \sim \pi_{k}(\cdot \mid x'_i)} \LB c\LP x'_i,d_{\theta}(x'_i)\RP \RB \\
        & = R(d_{\theta})
    \end{align*}

 Thus $\mathbb{E}_{k-1}[M_k] = M_{k-1}$, this concludes the lemma's proof.
   \end{proof}

Now; given a martingale is also a supermartingale, we can apply Ville's inequality which implies that with probability at least $1-\delta$, for any $k\geq 0$:

 \[ \mathbb{E}_{\theta\sim P}\exp(f_k(S,\theta)) \leq \frac{1}{\delta}. \]

 Plugging this back into \Cref{eq: change-meas-opl-subbatch-sec3} yields, with probability at least $ 1-\delta$ over $S$, for any $k\geq 0$, any posterior $Q$: 

 \begin{align*}
     \sum_{j=0}^k \mathbb{E}_{\theta\sim Q} \LB \sum_{i\in I_k} -\log\LP 1-\lambda R(d_{\theta}) \RP + \log\LP 1 - \lambda\frac{d_{\theta}(a_i\mid x_i) c_i}{\pi_{j}(a_i\mid x_i)} \RP \RB   \leq \operatorname{KL}(Q,P) + \log \left( \frac{1}{\delta}  \right).
 \end{align*}
Now remark that for any $j\geq0,i\in I_j,\theta$ because $d_{\theta}(a,x)= \mathds{1}\{f_{\theta}(x)= a\} \in \{0,1\}$, we have for any $(a,x)$:
\begin{align*}
    \log\LP 1 - \lambda\frac{d_{\theta}(a\mid x) c_i}{\pi_{j}(a\mid x)} \RP = d_{\theta}(a,x)\log\LP 1 - \lambda\frac{c_i}{\pi_{j}(a\mid x)} \RP.
\end{align*}

We then have: 
\begin{align*}
    \sum_{j=1}^k \mathbb{E}_{\theta\sim Q}\LB \sum_{i\in I_j}  - \log\LP 1 - \lambda\frac{d_{\theta}(a_i\mid x_i) c_i}{\pi_{j}(a_i\mid x_i)} \RP \RB & = -\sum_{j=0}^k \sum_{i\in I_j} \pi_{Q}(a_i\mid x_i)\log\LP 1 - \lambda\frac{c_i}{\pi_{j}(a_i\mid x_i)} \RP  \\
    & = (k+1)m\lambda \hat{R}_{0:k}^{\lambda-\mathrm{LS}}(\pi_{Q}).
\end{align*}

We also use the convexity of $x \rightarrow -\log(1 - \lambda x)$ to obtain: 
\begin{align*}
    \sum_{j=0}^k  \mathbb{E}_{\theta\sim Q} \LB \sum_{i\in I_j} -\log\LP 1-\lambda R(d_{\theta}) \RP \RB & = m (k+1)  \mathbb{E}_{\theta \sim Q} \LB -\log\LP 1-\lambda R(d_{\theta}) \RP \RB\\
    &\geq - m (k+1) \log\LP 1 -\lambda R(\pi_{Q}) \RP.
\end{align*}
Finally, mixing all those calculations together, dividing by $\lambda (k+1)m$ and applying the increasing function $\psi_\lambda(x) := \frac{1}{\lambda}(1-\exp(-\lambda x)$ gives: 

\[  R(\pi_Q) \leq \psi_\lambda\left( \hat{R}_{0:k}^{\lambda-\mathrm{LS}}(\pi_{\Q}) + \frac{\operatorname{KL}(Q,P) + \log \left( \frac{1}{\delta}  \right)}{\lambda (k+1)m}\right). \]

Finally, using that $\psi_\lambda(x) \leq x$ concludes the proof.
\end{proof}

\subsection{Proof of Proposition \ref{prop: slowconvergence}}
We prove the convergence guarantee of \Cref{alg:extension-sakhi}.

\slowconvergence*

\begin{proof}
     We fix $k\geq 0$, a sequential algorithm $\texttt{Alg}$ and $P$ a data-free prior. To highlight the dependency on $\lambda$, we denote in this proof $\mathcal{S}_{j} =\mathcal{S}_{\lambda,j}$ 
   Our goal is to apply the change of measure inequality on $\mathcal{H}$ to a specific function $f_k$. We define this function below, for any sample $S$ and any predictor $\theta$
   \begin{align*}
   f_k(S,\theta) := \sum_{j=0}^k \underbrace{\sum_{i\in I_k} -\lambda R(d_{\theta}) -\lambda^2 \Scal_{\lambda,j}\LP d_{\theta}\RP - \log\LP 1 - \lambda\frac{d_{\theta}(a_i\mid x_i) c_i}{\pi_{j}(a_i\mid x_i)} \RP}_{:= Y_j(S,\theta)},
   \end{align*}
    where $\mathcal{S}_{\lambda,j}(\pi) := \mathbb{E}_{(x,a,c)\sim P(\pi_j)}\left[\frac{ \pi(a \mid x) c^2}{\pi_{j}^2(a \mid x)-\lambda \pi_j(a\mid x) c}\right]$
   Now for a given posterior  $Q$ we apply the change of measure inequality for any $m$:
   \begin{align}
   \label{eq: change-meas-subbatch-subopt}
    \sum_{j=0}^k \mathbb{E}_{\theta\sim Q} \LB Y_j(S,\theta) \RB  = \mathbb{E}_{\theta\sim Q}\left[ f_k(S,\theta) \right] 
     \leq \operatorname{KL}(Q,P) + \log \left( \mathbb{E}_{\theta\sim P}\exp(f_k(S,\theta))  \right).
   \end{align}

   We need to control the exponential random variable. 
   To do so we prove the following lemma:

   \begin{lemma}
     The sequence $(M_k:=\mathbb{E}_{\theta\sim P}\exp(f_k(S,\theta))_{k\geq 0}$ is a non-negative supermartingale.
   \end{lemma}
   \begin{proof}
   We fix $k\geq 0$ and we recall that for any $k$, $P$ is data free. We show that $\mathbb{E}_{k-1}[M_k] \leq M_{k-1}$.

   \begin{align*}
       \mathbb{E}_{k-1}[M_k] & =\mathbb{E}_{k-1}\left[\mathbb{E}_{\theta\sim P}\exp(f_k(S,\theta)\right] \\
       &= \mathbb{E}_{\theta\sim P}\left[\mathbb{E}_{k-1}\exp(f_k(S,\theta)\right]\\
        & = \mathbb{E}_{\theta\sim P}\left[\mathbb{E}_{k-1}\left[\exp\LP f_{k-1}(S,\theta)\RP \exp\left( Y_j(S,\theta \right)\right] \right] \\
         & = \mathbb{E}_{\theta\sim P}\left[\exp\LP f_{k-1}(S,\theta)\RP \mathbb{E}_{k-1}\left[\exp\left( Y_j(S,\theta \right)\right] \right].
     \end{align*}
 The second line holding, as $P$ is data free, by Fubini.

    Now we note $c = c(a,x)$ for the sake of conciseness. For any $K, \theta\in \Rbb^d$, using the fact that conditionally to $\mathcal{F}_{k-1}$, all the data in $\S_k$ are independent we have:
    \begin{align*} 
    \Ebb_{k-1}\left[\exp\left( Y_j(S,\theta \right)\right] & =\prod_{i\in I_k}\exp \left(-\lambda\left(R\left(d_{\theta}\right)-\lambda \mathcal{S}_{\lambda,k}\left(d_{\theta}\right)\right)\right) \mathbb{E}_{k-1}\left[\frac{1}{1-\lambda \frac{d_{\theta}(a \mid x) c}{\pi_{k}(a \mid x)}}\right] 
    \intertext{ Using $\log(x)\leq x-1$ yields}
    & \leq \prod_{i\in I_k}\exp \left(-\lambda\left(R\left(d_{\theta}\right)-\lambda \mathcal{S}_{\lambda,k}\left(d_{\theta}\right)\right)+\mathbb{E}_{k-1}\left[\frac{1}{1-\lambda \frac{d_{\theta}(a \mid x) c}{\pi_{k}(a \mid x)}}\right]-1\right) \\ 
    & = \prod_{i\in I_k}\exp \left(-\lambda\left(R\left(d_{\theta}\right)-\lambda \mathcal{S}_{\lambda,k}\left(d_{\theta}\right)\right)+\mathbb{E}_{k-1}\left[\frac{\lambda d_{\theta}(a \mid x) c}{\pi_{k}(a \mid x)-\lambda d_{\theta}(a \mid x) c}\right]\right)
    \intertext{Using that $d_{\theta}$ is binary gives the equality $\frac{\lambda d_{\theta}(a \mid x) c}{\pi_{k}(a \mid x)-\lambda d_{\theta}(a \mid x) c}= \frac{\lambda d_{\theta}(a \mid x) c}{\pi_{k}(a \mid x)-\lambda c} $, thus}
    & = \prod_{i\in I_k}\left(\exp \left(-\lambda\left(R\left(d_{\theta}\right)-\lambda \mathcal{S}_{\lambda,k}\left(d_{\theta}\right)\right)+\mathbb{E}_{k-1}\left[\frac{\lambda d_{\theta}(a \mid x) c}{\pi_{k}(a \mid x)-\lambda c}\right]\right)\right) \\ 
    & =\prod_{i\in I_K}\left(\exp \left(-\lambda^2 \mathcal{S}_{\lambda,k}\left(d_{\theta}\right)+\mathbb{E}_{k-1}\left[\frac{\lambda d_{\theta}(a \mid x) c}{\pi_{k}(a \mid x)-\lambda c}-\lambda \frac{d_{\theta}(a \mid x) c}{\pi_{k}(a \mid x)}\right]\right)\right) \\ 
    & \leq \prod_{i\in I_k}\left(\exp \left(-\lambda^2 \mathcal{S}_{\lambda,k}\left(d_{\theta}\right)+\lambda^2 \mathcal{S}_{\lambda,k}\left(d_{\theta}\right)\right)\right)  \\ &
    \leq 1.\
    \end{align*}

    This allows concluding that $\Ebb_{k-1}[M_k] \leq M_{k-1}$, thus the proof.
\end{proof}

Now, we can apply Ville's inequality which implies that with probability at least $1-\delta$, for any $K\geq 1$:

 \[ \mathbb{E}_{\theta\sim P}\exp(f_k(S,\theta)) \leq \frac{1}{\delta}. \]

 Plugging this back into \Cref{eq: change-meas-subbatch-subopt} yields, with probability at least $ 1-\delta$ over $S$, for any $k\geq 0$, any $Q$: 

 \begin{multline*}
     \mathbb{E}_{\theta\sim Q} \LB\sum_{j=0}^k  \sum_{i\in I_k} -\lambda R(d_{\theta}) -\lambda \Scal_{\lambda,j}\LP d_{\theta}\RP - \log\LP 1 - \lambda\frac{d_{\theta}(a_i\mid x_i) c_i}{\pi_{j}(a_i\mid x_i)} \RP \RB  \\
     \leq \operatorname{KL}(Q,P) + \log \left( \frac{1}{\delta}  \right).
 \end{multline*}

Then recall that 

$$\hat{R}_{0:k}^{\lambda-\mathrm{LS}}(\pi_{Q}) := -\frac{1}{(k+1)m\lambda} \sum_{j=1}^{k} \sum_{i\in I_j} \pi_{Q}(a_i\mid x_i)\log\LP 1 - \lambda\frac{c_i}{\pi_{j}(a_i\mid x_i)} \RP.$$

We have, using again that $d_\theta$ is binary for any $\theta$: 
\begin{align*}
    \mathbb{E}_{\theta\sim Q}\LB\sum_{j=0}^k  \sum_{i\in I_k}  - \log\LP 1 - \lambda\frac{d_{\theta}(a_i\mid x_i) c_i}{\pi_{i}^0(a_i\mid x_i)} \RP \RB & = -\sum_{j=0}^k \sum_{i\in I_k} \pi_{Q}(a_i\mid x_i)\log\LP 1 - \lambda\frac{c_i}{\pi_{j}(a_i\mid x_i)} \RP  \\
    & = (k+1)m\lambda \hat{R}_{0:k}^{\lambda-\mathrm{LS}}(\pi_{Q}).
\end{align*}

Mixing all those calculations together and dividing by $\lambda (k+1)m$ gives: with probability at least $1-\delta$, for any $k$, $Q$: 

\begin{align}
    \label{eq: lower-bound-single-predictor}
    \hat{R}_{0:k}^{\lambda-\mathrm{LS}}(\pi_{Q}) \leq  R(\pi_{Q}) + \frac{\lambda}{k+1} \sum_{j=0}^k \Scal_{\lambda,j}\LP \pi_{Q}\RP   + \frac{\operatorname{KL}(Q,P) + \log \left( \frac{1}{\delta}  \right)}{\lambda (k+1)m}
\end{align}

    Now, we consider the optimal sequence, defined such that, $Q^\star:= \operatorname{argmin}_{Q} R(\Q)$. 

    We also consider $\pi_{k+1}$ the output of \Cref{alg:extension-sakhi} time $k$ satisfying:

    \[ \pi_{k+1} = \pi_{\hat{Q}_k},\quad \text{with}\;\hat{Q}_k := \operatorname{argmin}_{Q}  \hat{R}_{0:k}^{\lambda-\mathrm{LS}}(\pi_{Q}) + \frac{\operatorname{KL}(Q,P)}{\lambda (k+1)m} \]

    Then using \Cref{th: extension-sakhi}, and the definition of $\hat{Q}$ yields for any $\lambda$, data-free $P$ with probability at least $1-\delta/2$, for any $k$: 
    \begin{align*}
    R\LP \pi_{k+1} \RP & \leq \hat{R}_{0:k}^{\lambda-\mathrm{LS}}(\pi_{k+1}) + \frac{\operatorname{KL}(\hat{Q}_k,P) + \log \left( \frac{2}{\delta}  \right)}{\lambda (k+1)m} \\
    & \leq \hat{R}_{0:k}^{\lambda-\mathrm{LS}}(\pi_{\star}) + \frac{ \operatorname{KL}(Q^\star,P) + \log \left( \frac{2}{\delta}  \right)}{\lambda (k+1)m} 
\end{align*}

Then, exploiting \Cref{eq: lower-bound-single-predictor} with probability $1-\delta/2$ and taking a union bound yields, with probability at least $1-\delta$, for any $k$: 

\begin{align}
    \label{eq: final-subotpimality-single-predictor}
     R\LP \pi_{k+1} \RP - R\LP \pi_{\star} \RP & \leq \frac{\lambda}{k+1} \sum_{j=0}^k \Scal_{\lambda,j}\LP \pi_{\star}\RP + 2\frac{ \operatorname{KL}(Q^\star,P) + \log \left( \frac{2}{\delta}  \right)}{\lambda (k+1)m}. 
\end{align}
Finally, using that $\mathcal{S}_{\lambda,j} \leq \mathcal{S}_{j}$ for all $j$ concludes the proof.
\end{proof}




\section{SEQUENTIAL LEARNING ALGORITHMS WITH NON-IID CONTEXTS.}
\label{sec: non-iid-contexts}
In the main text, we always considered an \iid assumption on the contexts, which may be restrictive when considering real-life situations such as the behavior of an individual, susceptible to evolve through time. In what follows, \emph{we make no assumption on the context distribution and only assume that at time $k\in\{1\cdots,K\}$, conditionally to $\mathcal{F}_{k-1}$, $S_k= \LP (x_i,a_i,c_i) \RP_{i\in I_k}$ is constituted of mutually independent random variable}. We show that we are still able to obtain a generalization bound, thanks to the supermartingale toolbox of \citet{haddouche2022online,haddouche2023pac,chugg2023unified}. We first start with some additional framework. 

\paragraph{Framework} Assume that batches are of size $m$. We define, for all $k$, the risk at time $k$ of $\pi_k$ as 
$$R_k(\pi_k) = \mathbb{E}_{x'_k \sim \nu_k, a'_k \sim \pi_k(\cdot \mid x_k)}[c(x'_k, a'_k)],$$
where $\nu_k$ denotes, at time $k$, the conditional distribution of contexts with respect to the past (\ie $\mathcal{F}_{k-1}$) and $x'_k,a'_k$ are independent of $x_k,a_k$. We also define, for any $\pi_k,\pi_{k-1}$, the empirical risk at time $k$ as 
$$\hat{R}_k^{\lambda - LS}(\pi_k,\pi_{k-1}):= \frac{1}{m}\sum_{i\in I_k} h_{LS}^{\lambda-Lin}(\pi_k(a_i\mid x_i), \pi_{k-1}(a_i\mid x_i),c_i),$$ 
with $h_{LS}^{\lambda-Lin}$ defined in \Cref{sec: warmup-pls}. Given context distributions are evolving, we aim to learn a sequence of posteriors $Q_k$ which is $\mathcal{F}_k$- measurable from an initial prior $P=Q_0$ to make the cumulative risk $\frac{1}{K}\sum_{k=0}^K R_k(\pi_{Q_k})$ small. 

Adapting our proof technique yields the following result.

\begin{theorem}
\label{th: extension-sakhi-non-iid}
    For any subbatch size $m>0$, any $\lambda >0, \delta \in (0,1]$, any learning algorithm $\mathcal{A}$ such that for all $k$, $Q_k= \mathcal{A}(S_1,\cdots,S_k)$ with probability at least $1-\delta$ over $S$, for any $K\geq 1$:

    $$\frac{1}{K+1} \sum_{k=0}^K R_k(\pi_{Q_k}) \leq \frac{1}{K+1}\sum_{k=0}^K \hat{R}_k^{\lambda - LS}(\pi_{Q_k},\pi_{Q_{k-1}}) + \frac{\sum_{k=0}^K \operatorname{KL}(Q_k,Q_{k-1}) + \log \left( \frac{1}{\delta}  \right)}{\lambda (K+1)m}.$$
\end{theorem}  

This is dynamic counterpart of \Cref{th: extension-sakhi}, allowing shifting context distribution. Note that, similarly to \Cref{th: extension-sakhi}, this bound suggests a new learning algorithm detailed in \Cref{alg:extension-sakhi-non-iid}. This algorithm only exploits $S_k$ and $\pi_{Q_{k-1}}$ to create $\pi_{Q_k}$, contrary to \Cref{alg:extension-sakhi} which exploits all previous data $S_0\cup\cdots,\cup S_k$. The reason behind this discrepancy is that past data gathers outdated information as the context distribution may arbitrarily evolve. Then our theoretical result suggests to only consider current data. Note that the toolbox we use also allows an alternative result to \Cref{th: main-upperbound} for non-iid contexts, exploiting our novel regularized IPS. 

\RestyleAlgo{ruled}
\begin{algorithm}
\caption{\textbf{Sequential Policy Learning via Logarithmic Smoothing for shifting context distributions}}\label{alg:extension-sakhi-non-iid}
\textbf{Input}: Behavior policy $\pi_{Q_{-1}}$, a policy class $\mathcal{C}(\Theta)$, $\lambda >0$. \\
Initialize $P$ such that $\pi_{P} = \pi_{Q_{-1}}$ and $\mathcal{S}_{-1} = \{ \}$. \\
\For{$k \geq 0$}{
Collect data $S_k$ of size $|I_k| = m$ with $\pi_{k}$.\\
Solve the optimization problem:
\begin{align*}
\hat{Q}_{k+1} = \operatorname{argmin}_{Q\in\mathcal{C}(\Theta)} \left\{\hat{R}_k^{\lambda - LS}(\pi_{Q_k},\pi_{Q_{k-1}}) + \frac{\operatorname{KL}(Q_k,Q_{k-1})}{\lambda m} \right\}
\end{align*}
Set $\pi_{k+1} = \pi_{\hat{Q}_{k+1}}$.
}
\end{algorithm}

We then detail the proof of \Cref{th: extension-sakhi-non-iid}. 

\begin{proof}
    First, we follow the route described in \citet{haddouche2023pac}, Section 3. 

For any $K>0$, we look for a $K$-tuple of probabilities. Thus, our predictor set of interest is $\mathcal{H}_K:= \mathcal{H}^{\otimes K}$ and then, our predictor $\theta^K$ is a tuple $(\theta_1,..,\theta_K)\in\mathcal{H}_K$.

   Our goal is to apply the change of measure inequality on $\mathcal{H}_K$ to a specific function $f_K$. We define this function below, for any sample $S$ and any predictor $\theta^K=(\theta_1,...,\theta_K)$

   $$f_K(S,\theta^K) := \sum_{k=0}^K \underbrace{\sum_{i\in I_k} -\log\left( 1-\lambda R_k(d_{\theta_k}) \right) + \log\left( 1 - \lambda\frac{d_{\theta_k}(a_i\mid x_i) c_i}{\pi_{Q_{k-1}}(a_i\mid x_i)} \right)}_{:= Y_k(S,\theta_k)}.$$

   Now for a given posterior tuple $Q_1,...Q_K$ we define $Q^K= Q_1 \otimes ...\otimes Q_K$ and also $P^K = Q_{0}\otimes...\otimes Q_{K-1}$. We can now properly apply the change of measure inequality for any $m$:

\begin{align}
    \label{eq: change-meas-opl-subbatch-app}
    \sum_{k=0}^K \mathbb{E}_{\theta_k\sim Q_k} \left[ Y_k(S,\theta_k) \right] & = \mathbb{E}_{\theta^k\sim Q}\left[ f_K(S,\theta^K) \right] 
     \leq \operatorname{KL}(Q^K,P^K) + \log \left( \mathbb{E}_{\theta^K\sim P^K}\exp(f_K(S,\theta^K))  \right) \notag\\
     & = \sum_{k=0}^K \operatorname{KL}(Q_k,Q_{k-1}) +  \log \left( \mathbb{E}_{\theta^K\sim P^K}\exp(f_K(S,\theta^K))  \right),
\end{align}   

where we define $f_K(S,\theta^K):= \sum_{k=0}^K Y_k(S,\theta_k)$.

We need to control the exponential random variable. 
   To do so we prove the following lemma:

   \begin{lemma}
     The sequence $(M_k:=\mathbb{E}_{\theta^K\sim P^K}\exp(f_K(S,\theta^K))_{K\geq 0}$ is a non-negative martingale.
   \end{lemma}
   \begin{proof}
   We fix $K\geq 1$ and we recall that for any $k$, $Q_{k-1}$ is $\mathcal{F}_{k-1}$-measurable. We show that $\mathbb{E}_{K-1}[M_K] = M_{K-1}$.

     \begin{align*}
       \mathbb{E}_{K-1}[M_K] &=\mathbb{E}_{K-1}\left[\mathbb{E}_{\theta^K\sim P^K}\exp(f_K(S,\theta^K)\right] \\
        &= \mathbb{E}_{K-1}\left[\mathbb{E}_{\theta_1,..,\theta_K\sim Q_{0}\otimes...\otimes Q_{K-1}}\exp(f_K(S,\theta^K)\right] \\
        &= \mathbb{E}_{K-1}\left[\mathbb{E}_{\theta_1,..,\theta_K\sim Q_{0}\otimes...\otimes Q_{K-1}}\left[\exp\LP f_{K-1}(S,\theta^{K-1}) \RP \exp\left(\sum_{i\in I_K}\log\LP \frac{1 - \lambda\frac{d_{\theta_K}(a_i\mid x_i) c_i}{\pi_{Q_{K-1}}(a_i\mid x_i)}}{1-\lambda R_K(d_{\theta_K})}  \RP \right)\right] \right] \\
         &= M_{K-1} \mathbb{E}_{K-1}\left[ \mathbb{E}_{\theta_K\sim Q_{K-1}}\left[\Pi_{i\in I_K}\frac{1 - \lambda\frac{d_{\theta_K}(a_i\mid x_i) c_i}{\pi_{Q_{K-1}}(a_i\mid x_i)}}{1-\lambda R_K(d_{\theta_K})} \right]\right].
     \end{align*}
 The last line holding because $P^{K-1} = Q_{0}\otimes...\otimes Q_{K-2}$ is $\mathcal{F}_{K-1}$ measurable.

   Now we exploit the fact that $Q_{K-1}$ is $\mathcal{F}_{K-1}$ measurable to apply a conditional Fubini lemma stated in \citet[][Lemma D.3]{haddouche2022online}. We have:

   \begin{align*}
     \mathbb{E}_{K-1}\left[ \mathbb{E}_{\theta_K\sim Q_{K-1}}\left[ \Pi_{i\in I_K}\frac{1 - \lambda\frac{d_{\theta_K}(a_i\mid x_i) c_i}{\pi_{Q_{K-1}}(a_i\mid x_i)}}{1-\lambda R_K(d_{\theta_K})} \right]\right] & =  \mathbb{E}_{\theta_K\sim Q_{K-1}}\left[\mathbb{E}_{K-1}\left[\Pi_{i\in I_K}\frac{1 - \lambda\frac{d_{\theta_K}(a_i\mid x_i) c_i}{\pi_{Q_{K-1}}(a_i\mid x_i)}}{1-\lambda R_K(d_{\theta_K})} \right]\right]\\
     & = \mathbb{E}_{\theta_K\sim Q_{K-1}}\left[\Pi_{i\in I_K}\frac{1 - \lambda\Ebb_{K-1}\LB\frac{d_{\theta_K}(a_i\mid x_i) c_i}{\pi_{Q_{K-1}}(a_i\mid x_i)}\RB}{1-\lambda R_K(d_{\theta_K})} \right].
   \end{align*}
The last line holding as $R_K(d_{\theta_K})$ is $\mathcal{F}_{K-1}$-measurable for all $K,\theta_K$. This also holds as we exploited the assumption that, conditionally to $\mathcal{F}_{K-1}$, $S_K= \LP (x_i,a_i,c_i) \RP_{i\in I_K}$ is constituted of mutually independent random variable (thus an expectation of product equals a product of expectations). 
 Now we remark that, for any $\theta_K$, $i\in I_K$:
    \begin{align*}
        \mathbb{E}_{K-1}\left[\frac{d_{\theta_K}(a_i\mid x_i) c_i}{\pi_{Q_{K-1}}(a_i\mid x_i)}\right] & = \Ebb_{x'_i\sim \nu_K} \Ebb_{a'_i \sim \pi_{Q_{K-1}}(\cdot \mid x'_i)} \Ebb_{c'_i\sim p(\cdot\mid x'_i,a'_i)} \LB \frac{d_{\theta_K}(a'_i\mid x'_i) c'_i}{\pi_{Q_{K-1}}(a'_i\mid x'_i)} \RB \\
        & = \Ebb_{x'_i\sim \nu_K} \Ebb_{a'_i \sim \pi_{Q_{K-1}}(\cdot \mid x'_i)} \LB \frac{d_{\theta_K}(a'_i\mid x'_i)}{\pi_{Q_{K-1}}(a'_i\mid x'_i)} \Ebb_{c'_i\sim p(\cdot\mid x'_i,a'_i)} \LP c'_i\RP\RB \\
        & = \Ebb_{x'_i\sim \nu_K} \Ebb_{a'_i \sim \pi_{Q_{K-1}}(\cdot \mid x'_i)} \LB c\LP x'_i,d_{\theta_K}(x'_i)\RP \RB \\
        & = R_K(d_{\theta_K})
    \end{align*}

 Thus $\mathbb{E}_{K-1}[M_K] = M_{K-1}$, this concludes the lemma's proof.
   \end{proof}

Now; given a martingale is also a supermartingale, we can apply Ville's inequality which implies that with probability at least $1-\delta$, for any $k\geq 0$:

 \[ \mathbb{E}_{\theta^k\sim P^k}\exp(f_k(S,\theta)) \leq \frac{1}{\delta}. \]

 Plugging this back into \Cref{eq: change-meas-opl-subbatch-app} yields, with probability at least $ 1-\delta$ over $S$, for any $K\geq 0$, \begin{align*}
     \sum_{k=0}^K \mathbb{E}_{\theta_k\sim Q_k} \LB \sum_{i\in I_k} -\log\LP 1-\lambda R_k(d_{\theta_k}) \RP + \log\LP 1 - \lambda\frac{d_{\theta_k}(a_i\mid x_i) c_i}{\pi_{Q_{k-1}}(a_i\mid x_i)} \RP \RB   \leq \sum_{k=0}^K \operatorname{KL}(Q_k,Q_{k-1}) + \log \left( \frac{1}{\delta}  \right).
 \end{align*}
Now remark that for any $k\geq1,i\in I_k,\theta_k$, because $d_{\theta_k}(a,x)= \mathds{1}\{f_{\theta_k}(x)= a\} \in \{0,1\}$, we have for any $(a,x)$:
\begin{align*}
    \log\LP 1 - \lambda\frac{d_{\theta_k}(a\mid x) c_i}{\pi_{Q_{k-1}}(a\mid x)} \RP = d_{\theta_k}(a,x)\log\LP 1 - \lambda\frac{c_i}{\pi_{Q_{k-1}}(a\mid x)} \RP.
\end{align*}

We then have: 
\begin{align*}
    \sum_{k=0}^K \mathbb{E}_{\theta_k\sim Q_k}\LB \sum_{i\in I_k}  - \log\LP 1 - \lambda\frac{d_{\theta_i}(a_i\mid x_i) c_i}{\pi_{i}^0(a_i\mid x_i)} \RP \RB & = -\sum_{k=0}^K \sum_{i\in I_k} \pi_{Q_k}(a_i\mid x_i)\log\LP 1 - \lambda\frac{c_i}{\pi_{Q_{k-1}}(a_i\mid x_i)} \RP  \\
    & = m\lambda \sum_{k=0}^K \hat{R}_k^{\lambda - LS}(\pi_{Q_k}).
\end{align*}

We also use twice the convexity of $x \rightarrow -\log(1 - \lambda x)$ to obtain: 
\begin{align*}
    \sum_{k=0}^K  \mathbb{E}_{\theta_k\sim Q_k} \LB \sum_{i\in I_k} -\log\LP 1-\lambda R_k(d_{\theta_k}) \RP \RB & = m \sum_{k=0}^K \mathbb{E}_{\theta_k\sim Q_k} \LB -\log\LP 1-\lambda R_k(d_{\theta_k}) \RP \RB\\
    & \geq m \sum_{k=0}^K - \log\LP 1 -\lambda R_k(\pi_{Q_k}) \RP, \\
    & \geq -(K+1)m \log\LP 1 -\frac{\lambda}{K+1} \sum_{k=0}^K R_k(\pi_{Q_k}) \RP.
\end{align*}

Finally, mixing all those calculations together and dividing by $\lambda Km$ gives: with probability at least $1-\delta$, for any $K$: 

\begin{align*}
    -\frac{1}{\lambda} \log\LP 1 -\frac{\lambda}{K+1} \sum_{k=0}^K R_k(\pi_{Q_k}) \RP \leq \frac{1}{K+1}\sum_{k=0}^K \hat{R}_k^{\lambda - LS}(\pi_{Q_k},\pi_{Q_{k-1}}) + \frac{\sum_{k=0}^K \operatorname{KL}(Q_k,Q_{k-1}) + \log \left( \frac{1}{\delta}  \right)}{\lambda (K+1)m}
\end{align*}

Applying the function $\psi_\lambda$
Finally, mixing all those calculations together, dividing by $\lambda (k+1)m$ and applying the increasing function $\psi_\lambda(x) := \frac{1}{\lambda}(1-\exp(-\lambda x)$ gives: 

\[  \frac{1}{K+1} \sum_{k=0}^K R_k(\pi_{Q_k}) \leq \psi_\lambda\left( \frac{1}{K+1}\sum_{k=0}^K \hat{R}_k^{\lambda - LS}(\pi_{Q_k},\pi_{Q_{k-1}}) + \frac{\sum_{k=0}^K \operatorname{KL}(Q_k,Q_{k-1}) + \log \left( \frac{1}{\delta}  \right)}{\lambda (K+1)m}\right). \]

Finally, using that $\psi_\lambda(x) \leq x$ concludes the proof.
\end{proof}

\section{PROOFS OF SECTION \ref{sec: accelerated-convergence}}
\label{sec: proofs-sec-4}
\subsection{Proof of Theorem \ref{th: main-upperbound}}

\begin{proof}
For any $k \ge 0$, let:
\begin{align*}
   &Y_{i, j} := \log \LP 1-\lambda \left(R\LP d_\theta \RP -  R\LP \pi_j \right)\RP \RP -\log\left(1 - \lambda  \left(\frac{d_\theta(a_i|x_i)}{\pi_j(a_i|x_i)} - 1\right) c_i\right) \\
   &f_k(S,\theta) := - \sum_{j = 1}^k\sum_{i \in I_j} Y_{i, j}.
   \end{align*}
Now for a given posterior $Q$ and a prior $P$, we apply the change of measure inequality \citep{csizar1975divergence,donsker1976asymptotic} for any $k$:
   \begin{align}
   \label{eq: change-meas-opl-subbatch}
    \mathbb{E}_{\theta\sim Q}\left[ f_k(S,\theta) \right] 
     \leq \operatorname{KL}(Q,P) + \log \left( \mathbb{E}_{\theta\sim P}\exp(f_k(S,\theta))  \right).
   \end{align}

   We need to control the exponential random variable. 
   To do so we prove the following lemma:

   \begin{lemma}
     The sequence $(M_k:=\mathbb{E}_{\theta\sim P}\exp(f_k(S,\theta))_{k\geq 0}$ is a non-negative martingale.
   \end{lemma}
   \begin{proof}
   We fix $k\geq 0$. We show that $\mathbb{E}_{k-1}[M_k] = M_{k-1}$.
     \begin{align*}
       \mathbb{E}_{k-1}[M_k] &=\mathbb{E}_{k-1}\left[\mathbb{E}_{\theta \sim P}\exp(f_k(S,\theta)\right] \\
        &= \mathbb{E}_{\theta \sim P}\left[\mathbb{E}_{k-1}\left[\exp\LP f_{k-1}(S,\theta) \RP \exp\left(\sum_{i\in I_k}\log\LP \frac{1 - \lambda \left(\frac{d_{\theta}(a_i\mid x_i)}{\pi_{k}(a_i\mid x_i)} - 1\right)c_i}{1-\lambda (R(d_{\theta}) - R(\pi_k) )}  \RP \right)\right] \right] \\
        &= \mathbb{E}_{\theta \sim P}\left[\exp\LP f_{k-1}(S,\theta) \RP  \mathbb{E}_{k-1}\left[\exp\left(\sum_{i\in I_k}\log\LP \frac{1 - \lambda \left(\frac{d_{\theta}(a_i\mid x_i)}{\pi_{k}(a_i\mid x_i)} - 1\right)c_i}{1-\lambda (R(d_{\theta}) - R(\pi_k) )}  \RP \right)\right] \right].
     \end{align*}
 The second line comes from the application of Fubini as $P$ is data-independent. Now we have for any $\theta\in\Theta$:

   \begin{multline*}
     \mathbb{E}_{k-1} \left[ \exp\left(\sum_{i\in I_k}\log\LP \frac{1 - \lambda \left(\frac{d_{\theta}(a_i\mid x_i)}{\pi_{k}(a_i\mid x_i)} - 1\right)c_i}{1-\lambda (R(d_{\theta}) - R(\pi_k) )}  \RP \right)\right]  \\
     =  \mathbb{E}_{k-1}\left[\Pi_{i\in I_k}\frac{1 - \lambda \left(\frac{d_{\theta}(a_i\mid x_i)}{\pi_{k}(a_i\mid x_i)} - 1\right)c_i}{1-\lambda (R(d_{\theta}) - R(\pi_k) )} \right]\\
      = \Pi_{i\in I_k}\frac{1 - \lambda \Ebb_{k-1}\left[\left(\frac{d_{\theta}(a_i\mid x_i)}{\pi_{k}(a_i\mid x_i)} - 1\right)c_i\right]}{1-\lambda (R(d_{\theta}) - R(\pi_k) )},
   \end{multline*}
these computations holding as we exploited the assumption that, conditionally to $\mathcal{F}_{k-1}$, $S_k= \LP (x_i,a_i,c_i) \RP_{i\in I_k}$ is constituted of mutually independent random variables (thus an expectation of product equals a product of expectations). 
 Now, recall that $\pi_{k}$ is $\mathcal{F}_{k-1}$-measurable and that for all $i\in I_k, a_i \sim \pi_k(\cdot\mid x_i)$. We then remark that, for any $\theta$, $i\in I_k$:
    \begin{align*}
        \mathbb{E}_{k-1}\left[\frac{d_{\theta}(a_i\mid x_i) c_i}{\pi_{k}(a_i\mid x_i)}\right] & = \Ebb_{x'_i\sim \nu} \Ebb_{a'_i \sim \pi_{k}(\cdot \mid x'_i)} \Ebb_{c'_i\sim p(\cdot\mid x'_i,a'_i)} \LB \frac{d_{\theta}(a'_i\mid x'_i) c'_i}{\pi_{k}(a'_i\mid x'_i)} \RB \\
        & = \Ebb_{x'_i\sim \nu} \Ebb_{a'_i \sim \pi_{k}(\cdot \mid x'_i)} \LB \frac{d_{\theta}(a'_i\mid x'_i)}{\pi_{k}(a'_i\mid x'_i)} \Ebb_{c'_i\sim p(\cdot\mid x'_i,a'_i)} \LP c'_i\RP\RB \\
        & = \Ebb_{x'_i\sim \nu} \Ebb_{a'_i \sim \pi_{k}(\cdot \mid x'_i)} \LB c\LP x'_i,d_{\theta}(x'_i)\RP \RB \\
        & = R(d_{\theta})
    \end{align*}
 Similarly, notice that $\Ebb_{k-1}[c_i] = R(\pi_k)$.
 
 Thus $\mathbb{E}_{k-1}[M_k] = M_{k-1}$, this concludes the lemma's proof.
   \end{proof}

Now; given a martingale is also a supermartingale, we can apply Ville's inequality \citep{doob1939jean}, which implies that with probability at least $1-\delta$, for any $k\geq 0$:

 \[ \mathbb{E}_{\theta\sim P}\exp(f_k(S,\theta)) \leq \frac{1}{\delta}. \]

Plugging this into \Cref{eq: change-meas-opl-subbatch} yields, with probability at least $ 1-\delta$, for any $k\geq 0$, any $Q$: 

 \begin{align*}
     \sum_{j=1}^k \mathbb{E}_{\theta\sim Q} \LB \sum_{i\in I_j} -\log\LP 1-\lambda \left( R(d_{\theta}) - R(\pi_j)\right) \RP + \log\LP 1 - \lambda \left(\frac{d_{\theta}(a_i\mid x_i)}{\pi_{j}(a_i\mid x_i) } - 1\right) c_i\RP \RB   \leq \operatorname{KL}(Q,P) + \log \left( \frac{1}{\delta}  \right).
 \end{align*}
Now remark that for any $j\geq1,i\in I_j,\theta$ because $d_{\theta}(a,x)= \mathds{1}\{f_{\theta}(x)= a\} \in \{0,1\}$, we have for any $(a,x)$:
\begin{align*}
    \log\LP 1 - \lambda \left(\frac{d_{\theta}(a_i\mid x_i)}{\pi_{j}(a_i\mid x_i) } - 1\right) c_i\RP =  d_\theta(a_i|x_i)\log\LP 1 - \lambda\frac{c_i}{\pi_{j}(a_i\mid x_i)(1 + \lambda c_i)} \RP - \log(1 + \lambda c_i).
\end{align*}

We then have: 
\begin{align*}
    \sum_{j=1}^k \mathbb{E}_{\theta\sim Q}\LB \sum_{i\in I_j}  - \log\LP 1 - \lambda \left(\frac{d_{\theta}(a_i\mid x_i)}{\pi_{j}(a_i\mid x_i) } - 1\right) c_i\RP \RB & = N_k \lambda \left( \hat{R}^{\lambda-\mathrm{adj}}_{0:k}(\pi_Q) + \hat{C}_{0,k}(\lambda) \right).
\end{align*}

We also use the convexity of $x \rightarrow -\log(1 - \lambda x)$ to obtain: 
\begin{align*}
    \sum_{j=1}^k  \mathbb{E}_{\theta\sim Q} \LB \sum_{i\in I_j} -\log\LP 1-\lambda \left(R(d_{\theta}) - R(\pi_j) \right) \RP \RB &= - \sum_{j=1}^k  \mathbb{E}_{\theta\sim Q} \LB n_j \log \LP 1-\lambda \left( R(d_{\theta}) - R(\pi_j) \right) \RP \RB\\
    &= - N_k \sum_{j=1}^k  \mathbb{E}_{\theta\sim Q} \LB \frac{n_j}{N_k} \log\LP 1-\lambda \left(R(d_{\theta}) - R(\pi_j) \right) \RP \RB\\
    &\geq - N_k \log\LP 1 -\lambda \left( R(\pi_{Q}) - \sum_{j = 1}^k\frac{n_j}{N_k}R(\pi_j) \right) \RP.
\end{align*}
Finally, mixing all those calculations together, dividing by $\lambda N_k$ and applying the increasing function $\psi_\lambda(x) := \frac{1}{\lambda}(1-\exp(-\lambda x)$ gives: 

\[  R(\pi_Q) \leq \psi_\lambda\left( \hat{R}_{0,k}^{\lambda}(\pi_{Q}) + \hat{C}_{0,k}(\lambda) + \frac{\operatorname{KL}(Q,P) + \log \left( \frac{1}{\delta}  \right)}{\lambda N_k}\right). \]

Finally, using that $\psi_\lambda(x) \leq x$ concludes the proof.
\end{proof}

\subsection{Proof of Proposition \ref{prop: suboptimality_with_variance}}

We need a preliminary proposition to prove \Cref{prop: suboptimality_with_variance}.

After we proved our upper bound, we need another theorem to analyze the algorithm. 

\begin{proposition}
\label{prop: main-lowerbound}
    For any sequential algorithm $\texttt{Alg}$, data independent prior $P$, any $\boldsymbol{\lambda \in(0, 1)}, \delta \in (0,1]$, with probability at least $1-\delta$ over $S$, we have, for all $k\geq0$, denoting by $\pi_{k+1}$ the output of $\texttt{Alg}$ at time $k$, for all $Q\in\mathcal{P}(\Theta)$: 
\begin{align*}
     \hat{R}^{\lambda-\mathrm{adj}}_{0:k}(\pi_Q) + \hat{C}_{0,k}(\lambda) & \leq R\LP \pi_Q \RP - \sum_{j = 1}^k \frac{n_j}{N_k} R\LP \pi_j \RP  + \frac{\lambda}{1 - \lambda} \sum_{j = 1}^k\frac{n_j}{N_k} L\LP \pi_Q, \pi_j \RP  + \frac{ \operatorname{KL}(Q,P) + \log \left( \frac{1}{\delta}  \right)}{\lambda N_k}
\end{align*}
with $L\LP \pi_Q, \pi_j \RP = \mathbb{E}_{x \sim \nu}\left[ \mathbb{E}_{\pi_j}\left[c^2 \right] + \mathbb{E}_{\pi_Q}\left[c^2\left(\frac{1}{\pi_j(a|x)}
 - 2\right) \right] \right].$
\end{proposition}

\begin{proof}

For any $k \ge 0$, and $\lambda \in (0, 1)$, let:
\begin{align*}
   &Y_{i, j} :=  \log\left(1 - \lambda  \left(\frac{d_\theta(a_i|x_i)}{\pi_j(a_i|x_i)} - 1\right) c_i\right) +  \lambda \left( R\LP d_\theta \RP -  R\LP \pi_j \RP + \lambda M_\lambda(d_\theta, \pi_j)\right) \\
   &f_k(S,\theta) := - \sum_{j = 1}^k\sum_{i \in I_j} Y_{i, j}.
   \end{align*}
with $M_\lambda(d_\theta, \pi_j) = \mathbb{E}_{(x, a ,c ) \sim P(\pi_j)}\left[\frac{\left(\frac{d_\theta(a|x)}{\pi_j(a|x)} - 1 \right)^2c^2}{1 - \lambda \left(\frac{d_\theta(a|x)}{\pi_j(a|x)} - 1 \right)c} \right].$ Now for a given posterior $Q$ and a prior $P$, we apply the change of measure inequality \cite{csizar1975divergence,donsker1976asymptotic} for any $k$:
\begin{align}
   \label{eq: change-meas-opl-subbatch-2}
    \mathbb{E}_{\theta\sim Q}\left[ f_k(S,\theta) \right] 
     \leq \operatorname{KL}(Q,P) + \log \left( \mathbb{E}_{\theta\sim P}\exp(f_k(S,\theta))  \right).
   \end{align}

We also need to control the exponential random variable. We prove the following lemma:

   \begin{lemma}
     The sequence $(M_k:=\mathbb{E}_{\theta\sim P}\exp(f_k(S,\theta)))_{k\geq 0}$ is a non-negative supermartingale.
   \end{lemma}
   \begin{proof}
   We fix $k\geq 0$. We show that $\mathbb{E}_{k-1}[M_k] \le  M_{k-1}$.
   
   Let $L_k = \mathbb{E}_{k-1}\left[\exp\left(\sum_{i\in I_k}\log\LP \frac{1}{1 - \lambda \left(\frac{d_{\theta}(a_i\mid x_i)}{\pi_{k}(a_i\mid x_i)} - 1\right)c_i}\RP -  \lambda \left( R\LP d_\theta \RP -  R\LP \pi_k \RP + \lambda M_\lambda(d_\theta, \pi_k)\right)\right)\right].$ We have:
     \begin{align*}
       \mathbb{E}_{k-1}[M_k] &=\mathbb{E}_{k-1}\left[\mathbb{E}_{\theta \sim P}\exp(f_k(S,\theta)\right] \\
        &= \mathbb{E}_{\theta \sim P}\left[\exp\LP f_{k-1}(S,\theta) \RP  L_k \right],
     \end{align*}
 where the last line comes from Fubini's theorem as $P$ is data-independent. 
 Let us prove that $L_k$ is smaller than $1$. We have: 
 \begin{align*} 
  L_k &=\prod_{i\in I_k}\exp \left(-\lambda\left( R\left(d_{\theta}\right) - R\left(\pi_k\right) + \lambda M_\lambda\left(d_{\theta}, \pi_k\right)\right)\right) \mathbb{E}_{k-1}\left[\frac{1}{1-\lambda (\frac{d_{\theta}(a \mid x)}{\pi_{k}(a \mid x)} - 1)c }\right] 
    \intertext{ Using $\log(x)\leq x-1$ yields}
    L_k & \leq \prod_{i\in I_k}\exp \left(-\lambda\left( R\left(d_{\theta}\right) - R\left(\pi_k\right) + \lambda M_\lambda\left(d_{\theta}, \pi_k\right)\right) + \mathbb{E}_{k-1}\left[\frac{1}{1-\lambda (\frac{d_{\theta}(a \mid x)}{\pi_{k}(a \mid x)} - 1)c }\right] - 1\right)   \\
     & = \prod_{i\in I_k}\exp \left(-\lambda\left( R\left(d_{\theta}\right) - R\left(\pi_k\right) + \lambda M_\lambda\left(d_{\theta}, \pi_k\right)\right) + \mathbb{E}_{k-1}\left[\frac{\lambda (\frac{d_{\theta}(a \mid x)}{\pi_{k}(a \mid x)} - 1)c}{1-\lambda (\frac{d_{\theta}(a \mid x)}{\pi_{k}(a \mid x)} - 1)c }\right]\right)   \\
     & = \prod_{i\in I_k}\exp \left(-\lambda^2  M_\lambda\left(d_{\theta}, \pi_k\right) + \mathbb{E}_{k-1}\left[\frac{\lambda (\frac{d_{\theta}(a \mid x)}{\pi_{k}(a \mid x)} - 1)c}{1-\lambda (\frac{d_{\theta}(a \mid x)}{\pi_{k}(a \mid x)} - 1)c }- \lambda(\frac{d_{\theta}(a \mid x)}{\pi_{k}(a \mid x)} - 1)c \right] \right)   \\
     & = \prod_{i\in I_k}\exp \left(-\lambda^2  M_\lambda\left(d_{\theta}, \pi_k\right) + \mathbb{E}_{k-1}\left[\frac{\lambda^2 (\frac{d_{\theta}(a \mid x)}{\pi_{k}(a \mid x)} - 1)^2c^2}{1-\lambda (\frac{d_{\theta}(a \mid x)}{\pi_{k}(a \mid x)} - 1)c } \right] \right)   \\
     & = \prod_{i\in I_k}\exp \left(-\lambda^2  M_\lambda\left(d_{\theta}, \pi_k\right) + \lambda^2 M_\lambda\left(d_{\theta}, \pi_k\right)  \right)   \\
    &= 1.
    \end{align*}

Thus $\mathbb{E}_{k-1}[M_k] \le M_{k-1}$, this concludes the lemma's proof.
   \end{proof}

We can apply Ville's inequality \cite{doob1939jean} which implies that with probability at least $1-\delta$, for any $k\geq 0$:

 \[ \mathbb{E}_{\theta\sim P}\exp(f_k(S,\theta)) \leq \frac{1}{\delta}. \]

Plugging this into \Cref{eq: change-meas-opl-subbatch-2} yields, with probability at least $ 1-\delta$, for any $k\geq 0$, any $Q$: 

 \begin{align*}
     - \sum_{j=0}^k \sum_{i\in I_j} \mathbb{E}_{\theta\sim Q} \LB  Y_{i,j} \RB   \leq \operatorname{KL}(Q,P) + \log \left( \frac{1}{\delta}  \right).
 \end{align*}
with $Y_{i,j} = \log\left(1 - \lambda  \left(\frac{d_\theta(a_i|x_i)}{\pi_j(a_i|x_i)} - 1\right) c_i\right) +  \lambda \left( R\LP d_\theta \RP -  R\LP \pi_j \RP + \lambda M_\lambda(d_\theta, \pi_j)\right)$.

Now remark that for any $j\geq1,i\in I_j,\theta$ because $d_{\theta}(a,x)= \mathds{1}\{f_{\theta}(x)= a\} \in \{0,1\}$, we have for any $(a,x)$:
\begin{align*}
    \log\LP 1 - \lambda \left(\frac{d_{\theta}(a_i\mid x_i)}{\pi_{j}(a_i\mid x_i) } - 1\right) c_i\RP =  d_\theta(a_i|x_i)\log\LP 1 - \lambda\frac{c_i}{\pi_{j}(a_i\mid x_i)(1 + \lambda c_i)} \RP - \log(1 + \lambda c_i).
\end{align*}

We thus obtain by re-organising the terms and dividing by $\lambda N_k$:
\begin{align*}
    \hat{R}^{\lambda-\mathrm{adj}}_{0:k}(\pi_Q) + \hat{C}_{0,k}(\lambda) & \leq R\LP \pi_Q \RP - \sum_{j = 1}^k \frac{n_j}{N_k} R\LP \pi_j \RP  + \lambda \sum_{j = 1}^k\frac{n_j}{N_k} \mathbb{E}_{\theta \sim Q} \left[ M_\lambda\LP d_\theta, \pi_j \RP \right]  + \frac{ \operatorname{KL}(Q,P) + \log \left( \frac{1}{\delta}  \right)}{\lambda N_k}.
\end{align*}
We proceed then by bounding $\mathbb{E}_{\theta \sim Q} \left[ M_\lambda\LP d_\theta, \pi_j \RP \right]$:
\begin{align*}
    \mathbb{E}_{\theta \sim Q} \left[ M_\lambda\LP d_\theta, \pi_j \RP \right] &= \mathbb{E}_{\theta \sim Q} \left[ M_\lambda\LP d_\theta, \pi_j \RP \right]\\
    &= \mathbb{E}_{\theta \sim Q} \left[\mathbb{E}_{\pi_j}\left[\frac{\left(\frac{d_\theta(a|x)}{\pi_j(a|x)} - 1 \right)^2c^2}{1 - \lambda \left(\frac{d_\theta(a|x)}{\pi_j(a|x)} - 1 \right)c} \right] \right] \\
    &\le \frac{1}{1 - \lambda}\mathbb{E}_{\theta \sim Q} \left[\mathbb{E}_{\pi_j}\left[\left(\frac{d_\theta(a|x)}{\pi_j(a|x)} - 1 \right)^2c^2 \right] \right] \\
    &= \frac{1}{1 - \lambda}\mathbb{E}_{\pi_j}\left[\mathbb{E}_{\theta \sim Q} \left[\left(\frac{d_\theta(a|x)}{\pi_j(a|x)} - 1 \right)^2c^2 \right] \right] \\
    &= \frac{1}{1 - \lambda}\mathbb{E}_{\pi_j}\left[\mathbb{E}_{\theta \sim Q} \left[d_\theta(a|x)\left(\frac{1}{\pi_j(a|x)} - 1 \right)^2c^2 + (1 - d_\theta(a|x))c^2 \right] \right] \\
    &= \frac{1}{1 - \lambda}\mathbb{E}_{\pi_j}\left[\pi_Q(a|x)\left(\frac{1}{\pi_j(a|x)} - 1 \right)^2c^2 + (1 - \pi_Q(a|x))c^2 \right] \\
    &= \frac{1}{1 - \lambda}\mathbb{E}_{\pi_j}\left[\pi_Q(a|x)\left(\frac{1}{\pi_j(a|x)} - 1 \right)^2c^2 + (1 - \pi_Q(a|x))c^2 \right] \\
    &= \mathbb{E}_{\pi_j}\left[c^2 \right] + \mathbb{E}_{\pi_Q}\left[c^2\left(\frac{1}{\pi_j(a|x)}
 - 2\right) \right] = L(\pi_Q, \pi_j).
\end{align*}
We used several times above the property $d_\theta\in\{0,1\}$ to transform the analytical expression of the random variable with $\mathbb{E}_{\pi_j}$. This gives our result and ends the proof.
\end{proof}

Now we are able to prove our main result. Let us recall it below for conciseness. 

\begin{theorem}
    For any data independent prior $P$, any $\boldsymbol{\lambda \in(0, 1)}, \delta \in (0,1]$, with probability at least $1-\delta$, we have, denoting by $\pi_j$ the output of \Cref{alg:s_analyzed} at time $j$, for all $k\geq 0$:
\begin{align*}
      0 \le R\LP \pi_{k +1} \RP - R\LP \pi_{\star} \RP & \leq  \frac{\lambda}{1 - \lambda} \sum_{j = 1}^k\frac{n_j}{N_k} L\LP \pi_\star, \pi_j \RP  + 2\frac{ \operatorname{KL}(Q^\star,P) + \log \left( \frac{2}{\delta}  \right)}{\lambda N_k}
\end{align*}
with $\pi_{k+1}$ defined as the output of the sequential algorithm \Cref{alg:s_analyzed} at time $k$.
\end{theorem}

\begin{proof}
We know that by \Cref{th: main-upperbound}, for any data independent prior $P$, any $\boldsymbol{\lambda \in(0, 1)}, \delta \in (0,1]$, with probability at least $1-\delta/2$ over $S$, we have for all $k\geq 0$: 
\begin{align*}
     R\LP \pi_{k+1} \RP - \sum_{j = 1}^k \frac{n_j}{N_k} R\LP \pi_j \RP & \leq \psi_\lambda \left( \hat{R}^{\lambda-\mathrm{adj}}_{0:k}(\pi_{k+1}) + \hat{C}_{0,k}(\lambda) + \frac{ \operatorname{KL}(Q_{k+1},P) + \log \left( \frac{2}{\delta}  \right)}{\lambda N_k} \right) \\
     & \leq \hat{R}^{\lambda-\mathrm{adj}}_{0:k}(\pi_{k+1}) + \hat{C}_{0,k}(\lambda) + \frac{ \operatorname{KL}(Q_{k+1},P) + \log \left( \frac{1}{\delta}  \right)}{\lambda N_k},
\end{align*}
as $\psi_\lambda(x)\leq x$. Then, by definition, $\hat{Q}_{k+1} = \operatorname{argmin}_{Q} \left\{\hat{R}^{\lambda-\mathrm{adj}}_{0:k}(\pi_{\Q}) + \frac{\operatorname{KL}(Q,P)}{\lambda N_k} \right\}$, thus: 

\begin{align*}
 \forall k \ge 0, \quad    R\LP \pi_{k+1} \RP - \sum_{j = 1}^k \frac{n_j}{N_k} R\LP \pi_j \RP & \leq \hat{R}^{\lambda-\mathrm{adj}}_{0:k}(\pi_\star) + \hat{C}_{0,k}(\lambda) + \frac{ \operatorname{KL}(Q^\star,P) + \log \left( \frac{2}{\delta}  \right)}{\lambda N_k}.
\end{align*}

Now, we exploit \Cref{prop: main-lowerbound}, with probability at least $1-\delta/2$ over $S$, we have $\forall k \ge 0$: 
\begin{align*}
     \hat{R}^{\lambda-\mathrm{adj}}_{0:k}(\pi_\star) + \hat{C}_{0,k}(\lambda) & \leq R\LP \pi_\star \RP - \sum_{j = 1}^k \frac{n_j}{N_k} R\LP \pi_j \RP  + \frac{\lambda}{1 - \lambda} \sum_{j = 1}^k\frac{n_j}{N_k} L\LP \pi_\star, \pi_j \RP  + \frac{ \operatorname{KL}(Q^\star,P) + \log \left( \frac{2}{\delta}  \right)}{\lambda N_k}. 
\end{align*}

Now, taking an union bound and combining the two previous results gives, with probability at least $1-\delta$ over $S$, $\forall k \ge 0$:

\begin{align*}
R\LP \pi_{k+1} \RP - \sum_{j = 1}^k \frac{n_j}{N_k} R\LP \pi_j \RP & \leq R\LP \pi_\star \RP - \sum_{j = 1}^k \frac{n_j}{N_k} R\LP \pi_j \RP + \frac{\lambda}{1 - \lambda} \sum_{j = 1}^k\frac{n_j}{N_k} L\LP \pi_\star, \pi_j \RP \\ 
& \quad +  2\frac{ \operatorname{KL}(Q^\star,P) + \log \left( \frac{2}{\delta}  \right)}{\lambda N_k}.
\end{align*}

Re-organising the terms conclude the proof.
\end{proof}

\subsection{Proof of Lemma \ref{lemma:acc}} \label{app:convergence_lemma}

\begin{proof}
We are looking for $\gamma_k > 0$ such that:
\begin{align*}
     \frac{L(\pi_\star, \pi_k)}{R(\pi_k) - R(\pi_\star)} \le \gamma_k.
\end{align*}
Let us start rewriting the suboptimality $R(\pi_k) - R(\pi_\star)$. We have:
\begin{align*}
    R(\pi_k) - R(\pi_\star) &= \mathbb{E}_x\left[\mathbb{E}_{\pi_k}\left[c(a, x)\right] - c(a_\star(x), x)\right] \\
    &= \mathbb{E}_x\left[\mathbb{E}_{\pi_k}\left[c(a, x) - c(a_\star(x), x)\right]\right] \\
    &= \mathbb{E}_x\left[\sum_{a \neq a_\star(x)} \pi_{k}(a|x)\Delta_{a,x}\right].
\end{align*}
On the other hand, we have by definition of $L(\pi_\star, \pi_k)$:
\begin{align*}
     L(\pi_\star, \pi_k) &= \mathbb{E}_{x} \left[ \mathbb{E}_{\pi_k}\left[ \mathbb{E}_{p(\cdot|x, a)}\left[c^2 \right]\right] + \left(\frac{1}{\pi_k(a^\star(x)|x)} - 2 \right)\mathbb{E}_{p(\cdot|x, a_\star(x))}\left[c^2 \right]\right] \\
      &= \mathbb{E}_{x} \left[ \mathbb{E}_{\pi_k}\left[ \mathbb{E}_{p(\cdot|x, a)}\left[c^2 \right] - \mathbb{E}_{p(\cdot|x, a_\star(x))}\left[c^2 \right]\right] + \left(\frac{1}{\pi_k(a^\star(x)|x)} - 1 \right)\mathbb{E}_{p(\cdot|x, a_\star(x))}\left[c^2 \right]\right] \\
       &= \mathbb{E}_{x} \left[ \sum_{a \neq a_\star(x)}\pi_k(a|x)\left[ \mathbb{E}_{p(\cdot|x, a)}\left[c^2 \right] - \mathbb{E}_{p(\cdot|x, a_\star(x))}\left[c^2 \right]\right] + \left(\frac{1}{\pi_k(a^\star(x)|x)} - 1 \right)\mathbb{E}_{p(\cdot|x, a_\star(x))}\left[c^2 \right]\right] \\
       &= \mathbb{E}_{x} \left[ \sum_{a \neq a_\star(x)}\pi_k(a|x)\left[ \sigma^2(a,x) - \sigma^2(a_\star(x), x) + c^2(a, x) - c^2(a_\star(x), x)\right] \right]\\ 
        & + \mathbb{E}_{x}\left[\left(\frac{1}{\pi_k(a^\star(x)|x)} - 1 \right)\mathbb{E}_{p(\cdot|x, a_\star(x))}\left[c^2 \right]\right].
        \end{align*}
As the costs are in $[0, 1]$, we can bound the variance term by:
$$\mathbb{E}_{x} \left[ \sum_{a \neq a_\star(x)}\pi_k(a|x)\left[ \sigma^2(a,x) - \sigma^2(a_\star(x), x)\right] \right] \le \frac{1}{4} \mathbb{E}_{x} \left[ \sum_{a \neq a_\star(x)}\pi_k(a|x) \right] = \frac{1}{4} \mathbb{E}_{x} \left[ 1 -\pi_k(a_\star(x)|x) \right]. $$ 
We then use the definition of $\Delta_{x, a}$ to develop $c(a, x)^2$ and obtain:
\begin{align*}
    c(a, x)^2 &= \left(\Delta_{a, x} + c(a_\star(x), x) \right)^2 \\
    &= \Delta_{a, x}^2  + 2 c(a_\star(x), x)\Delta_{x,a} + c(a_\star(x), x)^2 \le \Delta_{a, x}^2  + c(a_\star(x), x)^2.
\end{align*}
Plugging these two results into our previous equation, we get:
\begin{align*}
     L(\pi_\star, \pi_k) &\le \frac{1}{4} \mathbb{E}_{x} \left[ 1 -\pi_k(a_\star(x)|x) \right] + \mathbb{E}_{x} \left[ \sum_{a \neq a_\star(x)}\pi_k(a|x) \Delta_{x, a}^2   \right] + \mathbb{E}_{x}\left[\left(\frac{1}{\pi_k(a^\star(x)|x)} - 1 \right)\mathbb{E}_{p(\cdot|x, a_\star(x))}\left[c^2 \right]\right] \\
     &\le \frac{1}{4} \mathbb{E}_{x} \left[ 1 -\pi_k(a_\star(x)|x) \right] + \mathbb{E}_{x} \left[ \sum_{a \neq a_\star(x)}\pi_k(a|x) \Delta_{x, a}^2   \right] + \mathbb{E}_{x}\left[\left(\frac{1}{\pi_k(a^\star(x)|x)} - 1 \right)\right]\\
     &\le \frac{1}{4} \mathbb{E}_{x} \left[ 1 -\pi_k(a_\star(x)|x) \right] + \mathbb{E}_{x} \left[ \sum_{a \neq a_\star(x)}\pi_k(a|x) \Delta_{x, a}^2   \right] + \frac{1}{C^\star_k}\mathbb{E}_{x}\left[1 - \pi_k(a^\star(x)|x)\right].
\end{align*}
We then proceed at upper bound the quantity:
\begin{align*}
     \frac{L(\pi_\star, \pi_k)}{R(\pi_k) - R(\pi_\star)} &\le \frac{\frac{1}{4} \mathbb{E}_{x} \left[ 1 -\pi_k(a_\star(x)|x) \right] + \mathbb{E}_{x} \left[ \sum_{a \neq a_\star(x)}\pi_k(a|x) \Delta_{x, a}^2   \right] + \frac{1}{C^\star_k}\mathbb{E}_{x}\left[1 - \pi_k(a^\star(x)|x)\right]}{\mathbb{E}_x\left[\sum_{a \neq a_\star(x)} \pi_{k}(a|x)\Delta_{a,x}\right]} \\
     &\le 1 + \left(\frac{1}{4} + \frac{1}{C_k^\star}\right)\frac{\mathbb{E}_{x} \left[ 1 -\pi_k(a_\star(x)|x) \right]}{\mathbb{E}_x\left[\sum_{a \neq a_\star(x)} \pi_{k}(a|x)\Delta_{a,x}\right]},
\end{align*}
because $\Delta_{x,a}^2 \le \Delta_{x,a}$ as $\Delta_{x,a} \le 1$. 

Now, we proceed to define:
$$\bar{\Delta}_k = \frac{\mathbb{E}_x\left[\sum_{a \neq a_\star(x)} \pi_{k}(a|x)\Delta_{a,x}\right]}{\mathbb{E}_{x} \left[ 1 -\pi_k(a_\star(x)|x) \right]}.$$
$\bar{\Delta}_k$ represents the expected sub-optimality incurred when policy $\pi_k$ selects a non-optimal action. A large $\bar{\Delta}_k$ means that the environment penalizes us more for not choosing an optimal action, while $\bar{\Delta}_k \approx 0$ means that the environment has multiple near-optimal actions. We can characterize $\gamma_k$ by this quantity, or loosen our bound and have a more interpretable condition. We choose the second approach. Let us suppose that there exists a $\Delta_u > 0$ for a $u \in [0, 1)$ such that:
\begin{align*}
    P_x\left(\min_{a \neq a^\star(x)} \Delta_{x,a} \ge \Delta_u \right) = 1 - u.
\end{align*}
Let us focus on the numerator of $\bar{\Delta}_k$. We have:
\begin{align*}
    P\left( \Delta_{x,a} \ge \Delta_u |a \neq a^\star(x)\right) \ge P_x\left(\min_{a \neq a^\star(x)} \Delta_{x,a} \ge \Delta_u \right) = 1-u. 
\end{align*}
Indeed, the event $E_1 =\{\Delta_{x,a} \ge \Delta_u |a \neq a^\star(x) \}$ represents choosing a random action from the suboptimal set, and finding that $\Delta_{x,a} \ge \Delta_u$, while the event $E_0 = \{\min_{a \neq a^\star(x)} \Delta_{x,a} \ge \Delta_u \}$ means that all non-optimal actions have a sub-optimality bigger than $\Delta_u$. We have $E_0 \subset E_1$, and thus:
\begin{align*}
    P\left( \Delta_{x,a} \ge \Delta_u |a \neq a^\star(x)\right) \ge   1-u. 
\end{align*}
Additionally, we have by the definition of conditional expectations:
\begin{align*}
    P\left( \Delta_{x,a} \ge \Delta_u |a \neq a^\star(x)\right) &= \frac{\mathbb{E}_x\left[\sum_{a \neq a_\star(x)} \pi_{k}(a|x)\mathds{1}\left[\Delta_{a,x} > \Delta_u \right]\right]}{\mathbb{E}_x\left[\sum_{a \neq a_\star(x)} \pi_{k}(a|x)\right]} \\
    &= \frac{\mathbb{E}_x\left[\sum_{a \neq a_\star(x)} \pi_{k}(a|x)\mathds{1}\left[\Delta_{a,x} > \Delta_u \right]\right]}{\mathbb{E}_{x} \left[ 1 -\pi_k(a_\star(x)|x) \right]} \ge 1 - u.
\end{align*}

Now, we write:
\begin{align*}
    \mathbb{E}_x\left[\sum_{a \neq a_\star(x)} \pi_{k}(a|x)\Delta_{a,x}\right] &\ge \mathbb{E}_x\left[\sum_{a \neq a_\star(x)} \pi_{k}(a|x) \Delta_{a,x}\mathds{1}\left[\Delta_{a,x} > \Delta_u \right]\right] \\
    &\ge \Delta_u\mathbb{E}_x\left[\sum_{a \neq a_\star(x)} \pi_{k}(a|x)\mathds{1}\left[\Delta_{a,x} > \Delta_u \right]\right] \\
    &\ge \Delta_u (1 - u) \mathbb{E}_{x} \left[ 1 -\pi_k(a_\star(x)|x) \right].
\end{align*}
Finally, we have:
\begin{align*}
    \bar{\Delta}_k = \frac{\mathbb{E}_x\left[\sum_{a \neq a_\star(x)} \pi_{k}(a|x)\Delta_{a,x}\right]}{\mathbb{E}_{x} \left[ 1 -\pi_k(a_\star(x)|x) \right]} \ge \Delta_u (1 - u).
\end{align*}
and we finally obtain our result: 
\begin{align*}
     \frac{L(\pi_\star, \pi_k)}{R(\pi_k) - R(\pi_\star)} &\le  1 + \left(\frac{1}{4} + \frac{1}{C_k^\star}\right)\frac{1}{\Delta_u(1 - u)} = \gamma_k
\end{align*}
This ends the proof.
\end{proof}

\subsection{Proof of Theorem \ref{th: suboptimality-conv-rate-unif-batch}}


We first start with a key intermediary result, derived from Lemma \ref{lemma:acc} coupled with \Cref{prop: suboptimality_with_variance}.
\begin{corollary}
    \label{cor: recursion}
    Let $\pi_{k+1}$ be the output of \Cref{alg:s_analyzed} at time $k$ with parameter $\lambda\in(0,1)$. For any data independent prior $P$, any $\boldsymbol{\lambda \in(0, 1)}, \delta \in (0,1]$, with probability at least $1-\delta$, we have for all $k \ge 0$:
\begin{align*}
      0 \le D_{k+1} & \leq  \frac{\lambda}{1 - \lambda} \sum_{j = 0}^k\frac{n_j}{N_k} \gamma_j D_j  + 2\frac{ \operatorname{KL}(Q^\star,P) + \log \left( \frac{2}{\delta}  \right)}{\lambda N_k},
\end{align*}
where $D_k = R(\pi_k) - R(\pi^\star)$.
\end{corollary}

We now proceed to prove Theorem \ref{th: suboptimality-conv-rate-unif-batch}.
\begin{proof}
    Assume that for all $j$, $n_j=m$, then $N_k= (k+1)m$ for all $k$.
     We re-use the notation $D_k =R(\pi_k) - R(\pi^\star) $ for all $k\geq 0$.
    We define 
    \begin{align*}
        \lambda_m &:= \frac{1 - \alpha}{8\gamma \sqrt{m} } \\
        C_\alpha &:= 64 \gamma \frac{\operatorname{KL}(\Q^\star,P) + \log(2/\delta)}{ (1 - \alpha)}.
    \end{align*}
    
Now let us prove the result by strong recursion. 
    
    \paragraph{Initialisation} Take $k=0$. We have, by \Cref{cor: recursion}:
    \begin{align*}
        D_1 & \leq \frac{\lambda}{1-\lambda}\gamma_1 D_0 + 2\frac{ \operatorname{KL}(Q^\star,P) + \log \left( \frac{2}{\delta}  \right)}{\lambda m}, \\
        & \leq \frac{\lambda}{1-\lambda}\gamma D_0 + 2\frac{ \operatorname{KL}(Q^\star,P) + \log \left( \frac{2}{\delta}  \right)}{\lambda m} 
        \intertext{Now we take $\lambda= \lambda_m $ and notice that $\lambda_m \le 1/2$, thus $\frac{\lambda_m}{1-\lambda_m} \leq 2\lambda_m$:}
        & \leq  \frac{1}{\sqrt{m}}\left( \frac{(1 - \alpha)}{4} D_0 + 16\gamma\frac{ \operatorname{KL}(Q^\star,P) + \log \left( \frac{2}{\delta}  \right)}{1 - \alpha} \right), \\
        & \leq  \frac{1}{\sqrt{m}}\left( \frac{(1 - \alpha)}{4} D_0 + 16\gamma\frac{ \operatorname{KL}(Q^\star,P) + \log \left( \frac{2}{\delta}  \right)}{1 - \alpha} \right), \intertext{Now as we have $\frac{1 - \alpha}{4}D_0 \le \frac{1}{4} \le \frac{C_\alpha}{2}$ we write:}
        & \leq \frac{1}{\sqrt{m}}\left( \frac{C_\alpha}{2} + \frac{C_\alpha}{4}\right),
    \end{align*}

This then ensures that $D_1 \leq \frac{C_\alpha}{\sqrt{m}}$ and prove the result for $k=1$
    
    \paragraph{General case.} Assume the result true for all $1\leq j\leq k$. By \Cref{cor: recursion}, we have for any $\lambda \le 1/2$: 

    \begin{align*}
     D_{k+1} & \leq  \frac{\lambda}{1 - \lambda} \frac{1}{k+1}\sum_{j = 0}^k \gamma_j D_{j}  + 2\frac{ \operatorname{KL}(Q^\star,P) + \log \left( \frac{2}{\delta}  \right)}{\lambda (k+1)m},\\
     & \leq 2\lambda \gamma \frac{1}{k+1}\sum_{j = 0}^k  D_{j}  + 2\frac{ \operatorname{KL}(Q^\star,P) + \log \left( \frac{2}{\delta}  \right)}{\lambda km},
     \intertext{Now, we use the recursion assumption to get}
      D_{k+1} & \leq 2\lambda \gamma\left( \frac{D_0}{k+1} +\frac{1}{(k+1)\sqrt{m}}\sum_{j = 1}^k  \frac{C_\alpha}{k^\alpha} \right)  + 2\frac{ \operatorname{KL}(Q^\star,P) + \log \left( \frac{2}{\delta}  \right)}{\lambda (k+1)m},\\
      & \leq 2\lambda \gamma\left( \frac{D_0}{k+1} +\frac{C_\alpha B(\alpha)}{(k+1)\sqrt{m}}k^{1-\alpha} \right)  + 2\frac{ \operatorname{KL}(Q^\star,P) + \log \left( \frac{2}{\delta}  \right)}{\lambda (k+1)m} \\
      & \leq 2\lambda \gamma\left( \frac{D_0}{k+1} +\frac{C_\alpha B(\alpha)}{(k+1)^{\alpha}\sqrt{m}} \right)  + 2\frac{ \operatorname{KL}(Q^\star,P) + \log \left( \frac{2}{\delta}  \right)}{\lambda (k+1)m},
      \intertext{Now as $\lambda_m \le 1/2$, take $\lambda = \lambda_m$, we have:}
      D_{k+1} & \leq \frac{1}{\sqrt{m}}\left( \frac{D_0(1 - \alpha)}{4(k+1)} +\frac{C_\alpha}{4(k+1)^{\alpha} \sqrt{m}} \right)  + 16\gamma\frac{ \operatorname{KL}(Q^\star,P) + \log \left( \frac{2}{\delta}  \right)}{ (1 - \alpha)(k+1)\sqrt{m}},\\
       & \leq \frac{1}{\sqrt{m}(k+1)^{\alpha}}\left( \frac{D_0(1 - \alpha)}{4} +\frac{C_\alpha}{4 \sqrt{m}}  + 16\gamma\frac{ \operatorname{KL}(Q^\star,P) + \log \left( \frac{2}{\delta}  \right)}{ (1 - \alpha)} \right),
      \intertext{We use $C_{\alpha} \geq D_0$ and its expression to get:}
      D_{k+1} & \leq \frac{1}{\sqrt{m}(k+1)^{\alpha}}\left( \frac{C_\alpha}{4} +\frac{C_\alpha}{4 \sqrt{m}}  + \frac{C_\alpha}{4} \right),
     \end{align*}

We then have: $D_{k+1} \leq \frac{C_\alpha}{\sqrt{m}(k+1)^\alpha}$, this concludes the proof.
\end{proof}

\subsection{Proof of Corollary \ref{cor: suboptimality-conv-rate}}

\begin{proof}[Proof of \Cref{cor: suboptimality-conv-rate}]   
    We re-use the notation $D_k =R(\pi_k) - R(\pi^\star) $ for all $ k\geq 0$.
    We define, for any $\alpha\in[0,1)$, 

    \[ C_\alpha := \max \left( D_0, (1-\lambda)\frac{\operatorname{KL}(\Q^\star,P) + \log(2/\delta)}{\lambda^2 \gamma^* \beta_2 B(\alpha)} \right) \leq \max \left( 1, \frac{\operatorname{KL}(\Q^\star,P) + \log(2/\delta)}{\lambda^2 \gamma^* \beta_2 B(\alpha)} \right), \]
    where $\lambda =\frac{1}{1 + 2^{2+\alpha}\gamma^\star \beta_1\beta_2 B(\alpha)}$
    Note that thanks to this definition, we have 
    \begin{align}
        \label{eq: temp_recursion}
        C_{\alpha}  \frac{\lambda}{1-\lambda}\gamma^* \beta_2B(\alpha) \geq \frac{\operatorname{KL}(\Q^\star,P)+\log(2/\delta)}{\lambda}.
    \end{align}
    Now let us prove the result by strong recursion. \paragraph{Initialisation.} By definition of $C_\alpha$, we have $D_0 \leq C_\alpha$.

    \paragraph{General case.}
    For $k\geq 0$, Assume the result true for all $j\leq k$. By \Cref{cor: recursion}, we have: 

    \begin{align*}
     D_{k+1} & \leq  \frac{\lambda}{1 - \lambda} \sum_{j = 0}^k\frac{n_j}{N_k} \gamma_j D_{j}  + 2\frac{ \operatorname{KL}(Q^\star,P) + \log \left( \frac{2}{\delta}  \right)}{\lambda N_k},\\
     & \leq \frac{\lambda}{1 - \lambda} \gamma^\star \beta_1\beta_2\frac{1}{k+1}\sum_{j = 0}^k D_{j}  + 2\frac{ \operatorname{KL}(Q^\star,P) + \log \left( \frac{2}{\delta}  \right)}{\lambda (k+1)} \beta_1.
     \intertext{In the last line, we used $\frac{k+1}{N_k}\leq \beta_1$, $\frac{n_j}{N_k} = \frac{n_j}{k+1}\frac{k+1}{N_k} \leq\frac{n_j}{k+1}\beta_1 \leq  \beta_1\frac{\beta_2}{k+1}$ and $\gamma_j\leq \gamma^{\star}$ for all $k\geq 0, j\leq k$. Now we use the recursion assumption.}
     D_{k+1} & \leq C_{\alpha} \frac{\lambda}{1 - \lambda} \gamma^\star \beta_1\beta_2\frac{1}{k+1}\sum_{j = 0}^k \frac{1}{(j+1)^{\alpha}}  + 2\frac{ \operatorname{KL}(Q^\star,P) + \log \left( \frac{2}{\delta}  \right)}{\lambda (k+1)} \beta_1.
     \intertext{Yet, $\sum_{j = 0}^k \frac{1}{(j+1)^{\alpha}} \leq 1 + B(\alpha)(k+1)^{1-\alpha} \leq 2B(\alpha)(k+1)^{1-\alpha}$, thus}
     D_{k+1} & \leq 2C_{\alpha} \frac{\lambda}{1 - \lambda} \gamma^\star \beta_1\beta_2B(\alpha)\frac{1}{(k+1)^\alpha}+ 2\frac{ \operatorname{KL}(Q^\star,P) + \log \left( \frac{2}{\delta}  \right)}{\lambda (k+1)} \beta_1. \\
     & \leq 2C_{\alpha} \frac{\lambda}{1 - \lambda} \gamma^\star \beta_1\beta_2B(\alpha)\frac{1}{(k+1)^\alpha} + 2\frac{ \operatorname{KL}(Q^\star,P) + \log \left( \frac{2}{\delta}  \right)}{\lambda (k+1)^\alpha} \beta_1.
     \intertext{Now, use that for all $k, (\frac{k+2}{k+1})^\alpha \leq 2^\alpha$:}
     D_{k+1} & \leq 2^{1+\alpha} \frac{\beta_1}{(k+2)^\alpha}\left(C_{\alpha} \frac{\lambda}{1 - \lambda} \gamma^\star \beta_2 B(\alpha)  + \frac{ \operatorname{KL}(Q^\star,P) + \log \left( \frac{2}{\delta}  \right)}{\lambda } \right).
     \intertext{Now using \Cref{eq: temp_recursion} gives:}
     D_{k+1} & \leq   \frac{\lambda}{1 - \lambda} 2^{2+\alpha}\gamma^\star \beta_1\beta_2 B(\alpha)   \frac{C_{\alpha}}{(k+2)^\alpha}.
\end{align*}

Now notice that $\lambda$ is such that

\[ \frac{\lambda}{1 - \lambda} 2^{2+\alpha}\gamma^\star \beta_1\beta_2 B(\alpha) = 1.  \]

This ensures that $D_{k+1} \leq \frac{C_{\alpha}}{(k+2)^\alpha}$ and concludes the recursion.   
\end{proof}


\section{EXPERIMENTAL DESIGN AND ADDITIONAL RESULTS} \label{app:experiments}

\subsection{Detailed Setup}

\paragraph{General Setup.} Our experimental framework follows the standard multiclass-to-bandit conversion widely used in prior work \citep{dudik14doubly, swaminathan2015batch}. Each multiclass dataset consists of feature-label pairs, which we transform into contextual bandit problems by treating features as contexts and labels as actions. The reward $r$ for taking action $a$ for context $x$ is modeled as Bernoulli with probability $p_x = \epsilon + \mathds{1}\left[a = a^\star(x)\right](1 - 2\epsilon)$, where $a^\star(x)$ is the true label of features $x$, and $\epsilon$ is a noise parameter, that we set to $\epsilon = 0.2$ in all our experiments. This setup ensures an average reward of $1 - \epsilon$ for the optimal action $a^\star(x)$ and $\epsilon$ for all others. In our experiments, we use parametrized Linear Gaussian Policies \citep{sakhi2023pac, aouali23a}, these policies interact with the contextual bandit environment to construct a logged bandit feedback dataset in the form $\{x_i, a_i, c_i\}_{i \in [n]}$, where $c_i = - r_i$ is the associated cost. 

In this procedure, we need three splits: $D_l$ (of size $n_l$) to train the logging policy $\pi_0$, another split $D_c$ (of size $n_c$) to generate the logging feedback with $\pi_0$, and finally a test split $D_{test}$ (of size $n_{test}$) to compute the true risk $R(\pi)$ of any policy $\pi$. In our experiments, we split the training split $D_{train}$ (of size $N$) of the four datasets considered into $D_l$ ($n_l = 0.05N$) and $D_c$ ($n_c = 0.95N$) and use their test split $D_{test}$. The detailed statistics of the different splits can be found in Table \ref{table:det_stats}. Recall that $K$ is the number of actions and $p$ the number of features. 

All datasets are used without processing aside from \texttt{CIFAR100} that we embed using a pretrained Resnet, obtaining a low dimensional vector of $100$ for each image. All experiments do not require heavy computations, they were conducted on a single 16 CPU/ 1 GPU V100 machine.

\begin{table}[h]
\centering
\begin{tabular}{ |c||c|c|c|c|c|c|}
 \hline
 Datasets & $N$ & $n_l$ & $n_c$ & $n_{test}$ & $K$ & $p$ \\
 \hline
 \textbf{MNIST}  & 60 000 & 3000 &  57 000 &  10 000  & 10 & 784\\
 \textbf{FashionMNIST} &  60 000 &  3000 &  57 000 &  10 000  & 10 & 784\\
 \textbf{EMNIST} & 112 800 &  5640 &  107 160 &  18 800  & 47 & 784\\
 \textbf{CIFAR100} & 60 000 & 3000 &  57 000 &  10 000  & 100 & 500\\
 \hline
\end{tabular}
\caption{Detailed statistics of the datasets used.}
\label{table:det_stats}
\end{table}

\paragraph{Linear Gaussian Policies} 
We use the Linear Gaussian Policy of \citep{sakhi2023pac}. To obtain these policies, we restrict $f_\theta$ to:
\begin{align}
    \forall x \in \mathcal{X}, \quad f_\theta(x) = \argmax_{a' \in \mathcal{A}} \left\{x^t\theta_{a'}\right\}
\end{align}
This results in a parameter $\theta$ of dimension $d = p \times K$ with $p$ the dimension of the features $\phi(x)$ and $K$ the number of actions. We also restrict the family of distributions $\mathcal{Q}_{d + 1} = \{Q_{\boldsymbol{\mu}, \sigma} = \mathcal{N}(\boldsymbol{\mu}, \sigma^2 I_d), \boldsymbol{\mu} \in \mathbb{R}^d, \sigma > 0\}$ to independent Gaussians with shared scale.  Estimating the propensity of $a$ given $x$ reduces the computation to a one dimensional integral:
\begin{align*}
    \pi_{\boldsymbol{\mu}, \sigma}(a|x) &= \mathbb{E}_{\epsilon \sim \mathcal{N}(0, 1)}\left[\prod_{a' \neq a} \Phi\left(\epsilon + \frac{\phi(x)^T(\boldsymbol{\mu}_a - \boldsymbol{\mu}_{a'})}{\sigma ||\phi(x)||}\right) \right]
\end{align*}
with $\Phi$ the cumulative distribution function of the standard normal.

\paragraph{The starting behavior policy $\pi_0$.} $\pi_0$ is trained on $D_l$ (supervised manner) with the following parameters: We use $L_2$ regularization of $10^{-4}$. This is used to prevent the logging policy $\pi_0$ from being close to deterministic, allowing efficient learning with importance sampling.  We use Adam \citep{kingma2014adam} with a learning rate of $10^{-1}$ for $10$ epochs. Once it is trained, we use an inverse temperature parameter $\alpha$ on its score to interpolate between a uniform policy $\alpha = 0$ and a trained policy $\alpha = 1$.

\paragraph{Optimizing our learning objectives.} In each optimization subroutine, we use Adam with a learning rate of $10^{-3}$ for $10$ epochs. The gradient of \textbf{LIG} policies is a one dimensional integral, and is approximated using $S = 32$ samples.
\begin{align*}
    \pi_{\boldsymbol{\mu}, \sigma}(a|x) &= \mathbb{E}_{\epsilon \sim \mathcal{N}(0, 1)}\left[\prod_{a' \neq a} \Phi\left(\epsilon + \frac{\phi(x)^T(\boldsymbol{\mu}_a - \boldsymbol{\mu}_{a'})}{\sigma ||\phi(x)||}\right) \right] \\
    &\approx \frac{1}{S} \sum_{s = 1}^S \prod_{a' \neq a} \Phi\left(\epsilon_s + \frac{\phi(x)^T(\boldsymbol{\mu}_a - \boldsymbol{\mu}_{a'})}{\sigma ||\phi(x)||}\right) \quad \epsilon_1, ..., \epsilon_S \sim \mathcal{N}(0, 1).
    \end{align*}

\subsection{More Experiments}

\paragraph{Effect of the quality of $\pi_0$.} In the main experiments, we used $\alpha = 0.2$, which makes the starting behavior policy keep some information about the good actions while being diffuse enough to put mass on all actions. We give here a more comprehensive view of how the algorithms behave depending on $\alpha$. The main take away, which reflects our first intuition, is that if we start far from $\pi_\star$, the improvement brought by sequential off-policy approaches is substantial compared to using the static behavior policy $\pi_0$ to collect data.

\textbf{Effect of Multiple Deployments while varying $\alpha$.} We focus on \texttt{CIFAR100} of large action space ($|\mathcal{A}| = 100$). $\alpha = 0.$ is the purely uniform behavior policy $\pi_0$, $\alpha = 0.5$ is a very good policy but is still diffuse and $\alpha = 1.$ is a peaked policy. Results are shown in \Cref{tab:alpha_effect}. \texttt{SeqAdjLS} outperforms \texttt{SeqLS} in all scenarios, with a bigger gap when $\pi_0$ is far from a good policy. 

\begin{table}
\vspace{-0.15cm}
    \centering
    \label{tab:alpha_effect}
    \begin{tabular}{c|c|c||c}
        \toprule
         $\alpha$ & $k$ & \texttt{SeqLS} (\texttt{Alg} \ref{alg:extension-sakhi}) & \texttt{SeqAdjLS} (\texttt{Alg} \ref{alg:s_analyzed}) \\
        \midrule
        \multirow{5}{*}{$\alpha = 0.$} & 
         $0$ & $-0.010$ & $-0.010$  \\
          & $1$ & $-0.011$ & $-0.011$ \\
          & $5$ & $-0.011$ & $-0.011$ \\
          & $10$ & $-0.020$ & $-0.026$  \\
          & $100$ & $-0.044$ & $\mathbf{-0.077}$  \\
        \cmidrule{1-4}
        \multirow{5}{*}{$\alpha = 0.5$} & 
           $0$ & $-0.268$ & $-0.268$  \\
          & $1$ & $-0.430$ & $-0.440$ \\
          & $5$ & $-0.500$ & $-0.520$ \\
          & $10$ & $-0.527$ & $-0.555$  \\
          & $100$ & $-0.583$ & $\mathbf{-0.621}$  \\
        \cmidrule{1-4}
        \multirow{5}{*}{$\alpha = 1.$} & 
           $0$ & $-0.531$ & $-0.531$  \\
          & $1$ & $-0.569$ & $-0.572$ \\
          & $5$ & $-0.590$ & $-0.600$ \\
          & $10$ & $-0.600$ & $-0.614$  \\
          & $100$ & $-0.630$ & $\mathbf{-0.649}$  \\
        \bottomrule
    \end{tabular}
    \caption{$R(\pi)$ varying $k$ and $\alpha$ in \texttt{CIFAR100} (lower is better).}
    \vspace{-0.35cm}
\end{table}

\textbf{Effect of $\alpha$ on \texttt{SCRM}.} We extend our evaluation of \texttt{SCRM} by varying the parameter $\alpha \in {0, 0.5, 1}$, with $k$ fixed at 10 across all datasets. When $\alpha = 0$, the behavior policy is entirely uniform, making the learning task considerably more difficult—particularly in environments with large action spaces—since the policy provides minimal guidance and infrequently selects low-cost actions. Despite these challenges, both \texttt{SeqLS} and \texttt{SeqAdjLS} consistently outperform \texttt{SCRM}. As $\alpha$ increases, the behavior policy becomes more informative, improving \texttt{SCRM}'s performance. Nonetheless, our sequential methods, especially \texttt{SeqAdjLS}, continue to achieve superior results. The accompanying plots highlight the consistent advantage of \texttt{SeqAdjLS} across all settings.

\begin{figure}[htbp]
    \centering
    \includegraphics[width=\textwidth]{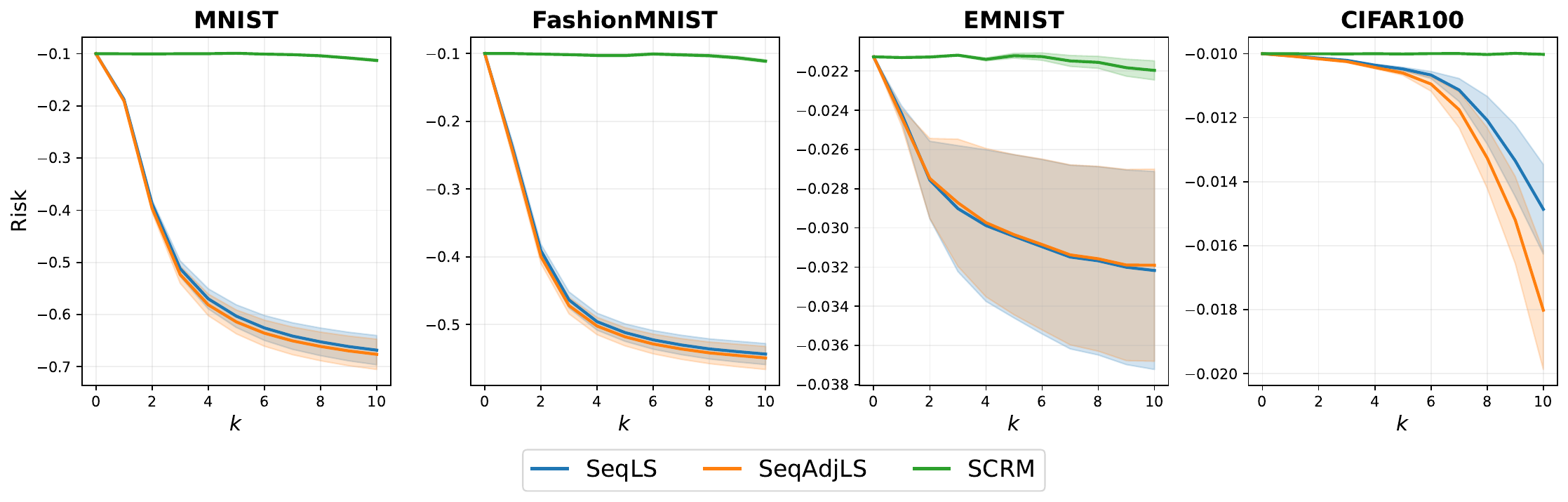}
    \caption{$R(\pi_k)$ at each intermediate step, for $k = 10$ and $\alpha = 0.$.}
    \label{fig:bench_alpha0}
\end{figure}

\begin{figure}[htbp]
    \centering
    \includegraphics[width=\textwidth]{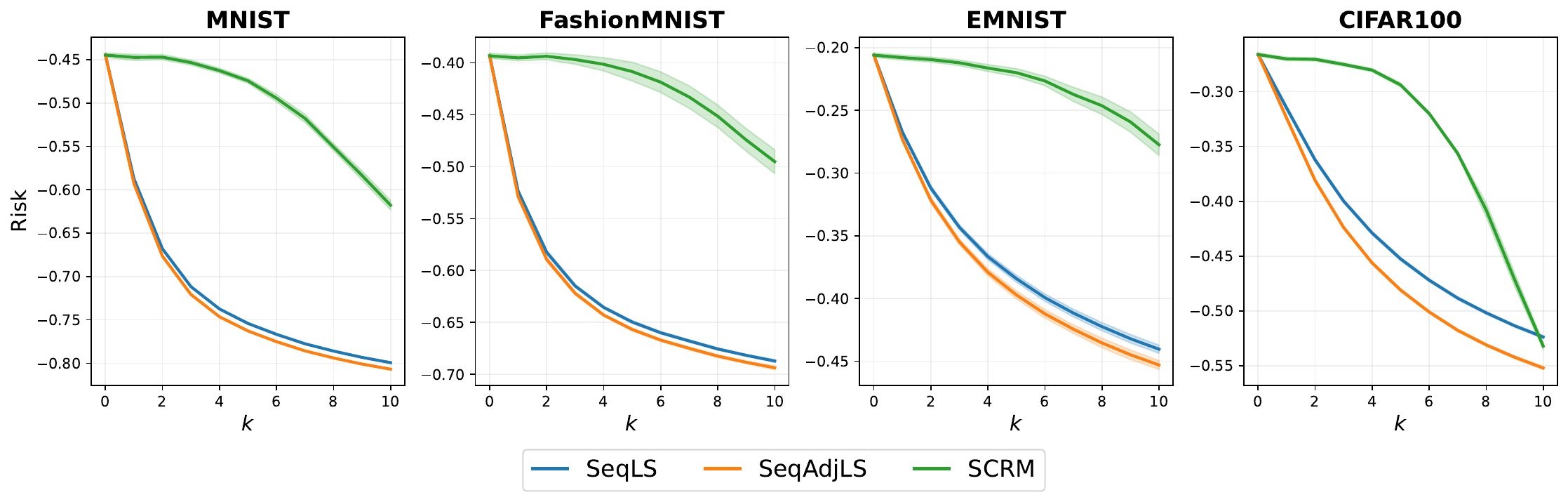}
    \caption{$R(\pi_k)$ at each intermediate step, for $k = 10$ and $\alpha = 0.5$.}
    \label{fig:bench_alpha0.5}
\end{figure}

\begin{figure}[htbp]
    \centering
    \includegraphics[width=\textwidth]{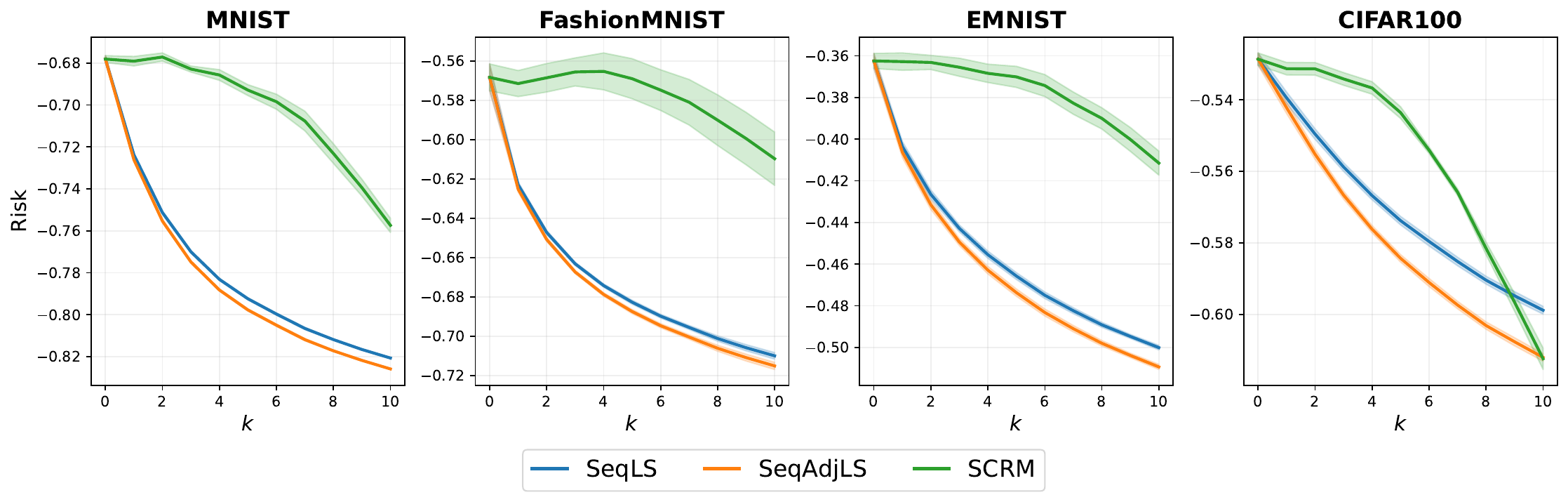}
    \caption{$R(\pi_k)$ at each intermediate step, for $k = 10$ and $\alpha = 1$.}
    \label{fig:bench_alpha1}
\end{figure}

\end{document}